\def\year{2022}\relax
\documentclass[letterpaper]{article} % DO NOT CHANGE THIS
\usepackage{aaai22}  % DO NOT CHANGE THIS
\usepackage{times}  % DO NOT CHANGE THIS
\usepackage{helvet}  % DO NOT CHANGE THIS
\usepackage{courier}  % DO NOT CHANGE THIS
\usepackage[hyphens]{url}  % DO NOT CHANGE THIS
\usepackage{graphicx} % DO NOT CHANGE THIS
\usepackage{natbib}  % DO NOT CHANGE THIS AND DO NOT ADD ANY OPTIONS TO IT
\usepackage{caption} % DO NOT CHANGE THIS AND DO NOT ADD ANY OPTIONS TO IT
\DeclareCaptionStyle{ruled}{labelfont=normalfont,labelsep=colon,strut=off} % DO NOT CHANGE THIS
\usepackage{microtype}
\usepackage{graphicx}
\usepackage{subfigure}
\usepackage{booktabs} % for professional tables
\usepackage{multirow}

\usepackage[ruled,linesnumbered]{algorithm2e}
\usepackage{etoolbox}
\makeatletter
\patchcmd{\@algocf@start}% <cmd>
  {-1.5em}% <search>
  {0pt}% <replace>
  {}{}% <success><failure>
\makeatother

\SetCommentSty{mycommfont}

\usepackage{color}
\usepackage{xcolor}
\usepackage{dsfont}
\usepackage{multirow}
\usepackage{balance} %to level the two last columns of refs
\usepackage{mathtools} %used here for adding comments at the end of equations
\usepackage{enumitem}
\usepackage{balance}
\usepackage{nicefrac}

\usepackage{amsthm}
\usepackage{amssymb}

\newtheorem{theorem}{Theorem}
\newtheorem*{theorem*}{Theorem}
\newtheorem{proposition}{Proposition}
\newtheorem{corollary}{Corollary}[theorem]
\newtheorem{lemma}[theorem]{Lemma}
\newtheorem*{lemma*}{Lemma}
\newtheorem{definition}{Definition}[section]

\theoremstyle{remark}
\newtheorem*{remark}{Remark}

\usepackage{bm}
\usepackage{bbm}
\usepackage{xspace}

\newcommand{\indic}[1]{\mathds{1}_{\!\{#1\}}}

\DeclareMathOperator*{\argmax}{argmax}

\newcommand{\norm}[1]{\left\lVert#1\right\rVert}

\newcommand{\intint}[1]{[#1]} %integer interval starting from 1
\newcommand{\card}[1]{|#1|}

\newcommand{\expe}[2]{\mathbb{E}_{#1}\left[{#2}\right]}

\newcommand{\expef}[1]{\mathbb{E}_{#1}}
\newcommand{\pr}[2]{\mathbb{P}_{#1}\left[{#2}\right]}

\newcommand{\group}{g}

\newcommand{\grarm}{k}
\newcommand{\user}{m}
\newcommand{\otheruser}{n}
\newcommand{\act}{a}
\newcommand{\ctx}{x}

\newcommand{\rew}{r}
\newcommand{\exprew}{\rho}
\newcommand{\util}{\mu}
\newcommand{\iutil}{u}
\newcommand{\cnt}{N}
\newcommand{\event}{\mathcal{E}}

\newcommand{\costexp}{C}
\newcommand{\conf}{\delta}

\newcommand{\narms}{K}
\newcommand{\nusers}{M}
\newcommand{\actS}{\mathcal A}
\newcommand{\actSsmall}{\mathcal{A}} %item groups
\newcommand{\ctxS}{\mathcal X}

\newcommand{\rad}{\beta}

\newcommand{\activepullS}{A}

\newcommand{\sumt}{\sum_{s=1}^t}
\newcommand{\sumCtx}{\sum_{\ctx \in \ctxS}}
\newcommand{\sumAct}{\sum_{\act \in \actS}}
\newcommand{\sumArm}{\sum_{\grarm=1}^\narms}

\newcommand{\sumActive}{\sum_{s \in \activepullS_{t-1}}}

\newcommand{\ucb}{\overline{\util}}
\newcommand{\lcb}{\underline{\util}}
\newcommand{\empEst}{\widehat{\util}}

\newcommand{\banditalg}{OCEF\xspace}
\newcommand{\bandialgcomplete}{OCEF (Online Certification of Envy-Freeness)\xspace}
\newcommand{\auditalg}{AUDIT\xspace}

\newcommand{\tmp}{f}

\newcommand{\piu}[1]{\pi^{#1, *}} %unconstrained
\newcommand{\pipar}[1]{\pi^{#1, {\rm par}}} %parity
\newcommand{\piee}[1]{\pi^{#1, {\rm eq}}} %equity of exposure

\newcommand{\euureg}{b}

\newcommand{\pieuu}{\pi^{{\rm euu}}_\euureg}

\newcommand{\degenvy}{\Delta}

\newcommand{\dpol}{D}

\newcommand{\nsampled}{\tilde{M}}
\newcommand{\efparams}{\epsilon,\gamma,\lambda}

\newcommand{\ceil}[1]{\left\lceil#1\right\rceil}
\newcommand{\nico}[1]{}%{\color{blue} N: #1}}
\newcommand{\virginie}[1]{}%{\color{purple}V: #1}}
\newcommand{\jamal}[1]{}%{\color{magenta} J: #1}}

\title{Online certification of preference-based fairness\\ for personalized recommender systems}
\author {
    Virginie Do\textsuperscript{\rm 1,2},
    Sam Corbett-Davies\textsuperscript{\rm 2},
    Jamal Atif\textsuperscript{\rm 1},
    Nicolas Usunier\textsuperscript{\rm 2}
}
\affiliations {
    \textsuperscript{\rm 1}LAMSADE, Université PSL, Université Paris Dauphine, CNRS, France\\
    \textsuperscript{\rm 2}Meta AI\\
    virginie.do@dauphine.eu, scd@fb.com, jamal.atif@lamsade.dauphine.fr, usunier@fb.com
}
\begin{document}

\maketitle
\begin{abstract}
Recommender systems are facing scrutiny because of their growing impact on the opportunities we have access to. Current audits for fairness are limited to coarse-grained parity assessments at the level of sensitive groups. We propose to audit for \emph{envy-freeness}, a more granular criterion aligned with individual preferences: every user should prefer their recommendations to those of other users. Since auditing for envy requires to estimate the preferences of users beyond their existing recommendations, we cast the audit as a new pure exploration problem in multi-armed bandits. We propose a sample-efficient algorithm with theoretical guarantees that it does not deteriorate user experience. We also study the trade-offs achieved on real-world recommendation datasets.
% %Personalized 
% Recommender systems are facing scrutiny because of their growing impact on the opportunities we have access to. Current audits for fairness are limited to coarse-grained parity assessments at the level of sensitive groups. We propose to audit for \emph{envy-freeness}, a more granular criterion aligned with individual preferences: every user should prefer their recommendations to those of other users. Since auditing for envy requires to estimate the preferences of users beyond their existing recommendations, we cast the audit as a new pure exploration problem in multi-armed bandits. We propose a sample-efficient
% %audit 
% algorithm with theoretical guarantees that it
% %exploration 
% does not deteriorate user experience. We also study 
% %empirically 
% the trade-offs achieved on real-world recommendation datasets.
\end{abstract}

\section{Introduction}

Recommender systems shape the information and opportunities available to us, as they help us prioritize content from news outlets and social networks, sort job postings, or find new people to connect with. To prevent the risk of unfair delivery of opportunities across users, substantial work has been done to audit recommender systems \cite{sweeney2013discrimination,asplund2020auditing,imana2021auditing}. %\cite{sweeney2013discrimination,lambrecht2019algorithmic,asplund2020auditing,imana2021auditing}. 
For instance, \citet{datta2015automated} found that women received fewer online ads for high-paying jobs than equally qualified men, while \citet{imana2021auditing} observed different delivery rates of ads depending on gender for different companies proposing similar jobs. 

%The art of performing the audits above is to carefully control for possible factors that could explain disparities in item recommendation rates across different user groups, such as education level or occupation in job recommendation audits. Yet, in general, 

The audits above aim at controlling for the possible acceptable justifications of the disparities, such as education level  in job recommendation audits. Yet, the observed disparities in recommendation do not necessarily imply that a group has a less favorable treatment: they might as well reflect that individuals of different groups tend to prefer different items. To strengthen the conclusions of the audits, it is necessary to develop methods that account for user preferences. Audits for equal satisfaction between user groups follow this direction  \citep{mehrotra2017auditing}, but they also have limitations. For example, they require interpersonal comparisons of measures of satisfaction, a notoriously difficult task \citep{sen1999possibility}. 

We propose an alternative approach to incorporating user preferences in audits which focuses on \emph{envy-free recommendations}: the recommender system is deemed fair if each user prefers their recommendation to those of all other users. Envy-freeness allows a system to be fair even in the presence of disparities between groups as long as these are justified by user preferences. On the other hand, if user B systematically receives better opportunities than user A \emph{from A's perspective}, the system is unfair. The criterion does not require interpersonal comparisons of satisfaction, since it relies on comparisons of different recommendations from the perspective of the same user. Similar fairness concepts have been studied in classification tasks under the umbrella of preference-based fairness \citep{zafar2017parity,kim2019preference,ustun2019fairness}. Envy-free recommendation is %thus 
the extension of these approaches to personalized recommender systems.

Compared to auditing for recommendation parity or equal satisfaction, auditing for envy-freeness poses new challenges. First, envy-freeness requires answering counterfactual questions such as ``would user A get higher utility from the recommendations of user B than their own?'', while searching for the users who most likely have the best recommendations from A's perspective. This type of question can be answered reliably only through active exploration, hence we cast it in the framework of pure exploration bandits  \cite{bubeck2009pure}. To make such an exploration possible, we consider a scenario where the auditor is allowed to replace a user’s recommendations with those that another user would have received in the same context. Envy, or the absence thereof, is estimated by suitably choosing whose recommendations should be shown to whom. While this scenario is more intrusive than some black-box audits of parity, auditing for envy-freeness provides a more compelling guarantee on the wellbeing of users subject to the recommendations.

The second challenge is that active exploration requires randomizing the recommendations, which in turn might alter the user experience. In order to control this cost of the audit (in terms of user utility), we follow the framework of conservative exploration \cite{wu2016conservative,garcelon2020conservative}, which guarantees a performance close to the audited system. We provide a theoretical analysis of the trade-offs that arise, in terms of the cost and duration of the audit (measured in the number of timesteps required to output a certificate).

Our technical contributions are twofold. \textbf{(1)} We provide a novel formal analysis of envy-free recommender systems, including a comparison with existing item-side fairness criteria and a probabilistic relaxation of the criterion. \textbf{(2)} We cast the problem of auditing for envy-freeness as a new pure exploration problem in bandits with conservative exploration constraints, and propose a sample-efficient auditing algorithm which provably maintains, throughout the course of the audit, a performance close to the audited system.

We discuss the related work in Sec. \ref{sec:related}. Envy-free recommender systems are studied in Sec. \ref{sec:envy}. In Sec. \ref{sec:algo}, we present the bandit-based auditing algorithm. In Sec. \ref{sec:exps}, we investigate the trade-offs achieved on real-world datasets.% by this algorithm on real-world recommendation datasets.

\section{Related work}\label{sec:related}

\paragraph{Fair recommendation}
The domain of fair machine learning is organized along two orthogonal axes. The first axis is whether fairness is oriented towards groups defined by protected attributes \cite{barocas2016big}, or rather oriented towards individuals \cite{dwork2012fairness}. The second axis is whether fairness is a question of \emph{parity} (predictions [or prediction errors] should be invariant by group or individual) \cite{corbett2018measure,kusner2017counterfactual}, or \emph{preference-based} (predictions are allowed to be different %as long as
if they faithfully reflect the preferences of all parties) \cite{zafar2017parity,kim2019preference,ustun2019fairness}. Our work takes the perspective of envy-freeness, which follows the preference-based approach and is aimed towards individuals.

The literature on fair recommender systems covers two problems: \emph{auditing} existing systems, and \emph{designing} fair recommendation algorithms. Most of the \emph{auditing} literature focused on group parity in recommendations \cite{hannak2014measuring,lambrecht2019algorithmic}, and equal user utility \cite{mehrotra2017auditing,ekstrand2018all}, while our audit for envy-freeness focuses on whether personalized results are aligned with (unknown) user preferences. On the \emph{designing} side,  \citet{patro2020fairrec,ilvento2020multi} cast fair recommendation as an allocation problem, with criteria akin to envy-freeness. They do not address the partial observability of preferences, so they cannot guarantee user-side fairness 
without an additional certificate that the estimated preferences effectively represent the true user preferences. Our work is thus complementary to theirs.

While we study fairness for users, recommender systems are multi-sided \cite{burke2017multisided, patro2020fairrec}, thus fairness can also be oriented towards recommended items \cite{celis2017ranking, biega2018equity,geyik2019fairness}. %Envy-freeness is a user-oriented notion which can be considered jointly with item-side fairness notions.

\paragraph{Multi-armed bandits} In \emph{pure exploration} bandits \cite{bubeck2009pure,audibert2010best}, an agent has to identify a specific set of arms after exploring as quickly as possible, without performance constraints. Our setting is close %to multiple testing \cite{jamieson2018bandit} and 
to threshold bandits \cite{locatelli2016optimal,kano2019good} where the goal is to find arms with better performance than a given baseline. Outside pure exploration, in the \emph{regret minimization} setting, conservative exploration \cite{wu2016conservative} enforces the anytime average performance to be not too far worse than that of a baseline arm.

In our work, the baseline is \emph{unknown} -- it is the current recommender system -- and the other ``arms'' are other users' policies. The goal is to make the decision as to whether an arm is better than the baseline, while not deteriorating performance compared to the baseline. We thus combine pure exploration and conservative constraints.

Existing work on fairness in exploration/exploitation \cite{joseph2016fairness,jabbari2017fairness,liu2017calibrated} is different from ours because unrelated to personalization. 

\paragraph{Fair allocation}
Envy-freeness was first studied in fair allocation
\cite{foley1967resource} in social choice.  Our setting is different because: a) the same item can be given to an unrestricted number of users, %there are no resource constraints, 
and b) true user preferences are unknown. %There are several notions of group envy-freeness in fair division, either considering every subset of individuals (e.g., \citet{berliant1992fair}), or pre-defined groups (e.g., \citet{manurangsi2017asymptotic}), but they are %significantly
%conceptually different from ours. %Unlike ours, they do not relax the individual notion. 

% Notice that while the original frameworks of preference-based fairness for classification are defined at the level of groups \cite{zafar2017parity,ustun2019fairness}, it is more challenging to define group envy-freeness when the recommendations are personalized. This is because in the original setting, there is only one classifier per group, while in our case we have a recommendation policy per individual in the group. In our setting, we need a non-trivial definition to capture what it means for a group of users to be ``envious of the recommendations of another group'', since there is no single group-level recommendation. 

\section{Envy-free recommendations}\label{sec:envy}
\subsection{Framework}
There are $\nusers$ users, and we identify the set of users with $\intint{\nusers}=\{1, \ldots, \nusers\}$. A personalized recommender system has one stochastic recommendation policy $\pi^\user$ per user $\user$. We denote by $\pi^\user(\act|\ctx)$ the probability of recommending item $\act\in\actS$ for user $\user\in\intint{\nusers}$ in context $\ctx\in\ctxS$.  We assume that $\ctxS$ and $\actS$ are finite to simplify notation, but this has no impact on the results. We consider a synchronous setting where at each time step $t$, the recommender system observes a context $\ctx^\user_t\sim q^\user$ for each user, selects an item $\act^\user_t\sim\pi^\user(.|\ctx^\user_t)$ and observes reward $\rew^\user_t\sim\nu^\user(\act^\user_t|\ctx^\user_t)\in[0,1]$.  We denote by $\exprew^\user(\act|\ctx)$ the expected reward for user $\user$ and item $\act$ in context $\ctx$, and, for any recommendation policy $\pi$, $\iutil^\user(\pi)$ is the utility of $m$ for $\pi$:
\begin{equation}
\begin{aligned}
\iutil^\user(\pi) &= \expef{\ctx\sim q^\user}\expef{\act \sim \pi(.|\ctx)}\expe{r \sim \nu^\user(\act|\ctx)}{r} \label{eq:userutil}\\
&=  \sumCtx \sumAct  q^\user(\ctx) \pi(\act|\ctx) \exprew^\user(\act|\ctx)
\end{aligned}
\end{equation}
We assume that the environment is \emph{stationary}: the context and reward distributions $q^\user$ and $\nu^\user$, as well as the policies $\pi^\user$ are fixed. Even though in practice policies evolve as they learn from user interactions and user needs change over time, we leave the study of non-stationarities for future work. The stationary assumption approximately holds when these changes are slow compared to the time horizon of the audit, which is reasonable when significant changes in user needs or recommendation policies take e.g., weeks.
Our approach %compares recommendation policies between users, so it 
applies when items $\act$ are single products as well as when items are structured objects such as rankings. Examples of (context $\ctx$, item $\act$) pairs include: $\ctx$ is a query to a search engine and $\act$ is a document or a ranking of documents, or $\ctx$ is a song chosen by the user and $\act$ a song to play next or an entire playlist. Remember, our goal is \emph{not} to learn the user policies $\pi^\user$, but rather to audit existing $\pi^\user$s for fairness.

% \paragraph{Notation} $\card{\mathcal S}$ denotes the cardinal of a finite set ${\mathcal S}$. $\norm{.}_1$ denotes the 1-norm. For an integer $N$, $[N]  =\{1,\ldots,N\}$.

\subsection{$\epsilon$-envy-free recommendations}\label{sec:envy-def}

Existing audits for user-side fairness in recommender systems are based on two main criteria:
\begin{enumerate}[leftmargin=*]
    \item \emph{recommendation parity}: the distribution of recommended items should be equal across (groups of) users,
    \item \emph{equal user utility}: all (groups of) users should receive the same utility, i.e. $\forall m,n, \, u^m(\pi^m) = u^n(\pi^n).$
\end{enumerate}
% Existing fairness audits of recommender systems are based on two main criteria (here stated at the individual level for clarity, but which were studied as averages over groups):
% \begin{enumerate}[leftmargin=*]
%     \item \emph{recommendation parity}: the distribution of recommended items should be equal across users, i.e. $\forall m,n,\,\pi^m = \pi^n$,
%     \item \emph{equal user utility}: all users should receive the same utility, i.e. $\forall m,n, \, u^m(\pi^m) = u^n(\pi^n).$
% \end{enumerate}
There are two ways in which these criteria conflict with the goal of personalized recommender systems to best accomodate user preferences. First, recommendation parity does not control for disparities that are aligned with user preferences. Second, equal user utility drives utility down as soon as users have different best achievable utilities. To address these shortfalls, we propose envy-freeness as a complementary diagnosis for the fairness assessment of personalized recommender systems. In this context, envy-freeness requires that users prefer their recommendations to those of any other user: 

\begin{definition}\label{dfn:envyfree}
Let $\epsilon\!\geq\!0$. 
A recommender system is \emph{$\epsilon$-envy-free} if:     $\quad \forall \user, \otheruser \in [\nusers]: \quad\iutil^\user(\pi^{\otheruser}) \leq \epsilon + \iutil^\user(\pi^{\user}).$
\end{definition}

Envy-freeness, originally studied in fair allocation \cite{foley1967resource} and more recently fair classification \cite{balcan2018envy,ustun2019fairness,kim2019preference}, stipulates that it is fair to apply different policies to different individuals or groups as long as it benefits everyone. Following this principle, we consider the personalization of recommendations as fair only if it better accommodates individuals' preferences. In contrast, we consider unfair the failure to give users a better recommendation when one such is available to others.

Unlike parity or equal utility, envy-freeness is in line with giving users their most preferred recommendations (see Sec. \ref{sec:compatibility}). Another improvement from equal user utility is that it does not involve interpersonal utility comparisons.

% Envy-freeness is in line with giving users their most preferred recommendations (see Sec. \ref{sec:compatibility}), unlike the two previous criteria. Another improvement from equal user utility is that it does not involve interpersonal utility comparisons, which spares us from comparing measures of performance that are scaled differently across users (due to e.g., different rating scales or browsing patterns).

% Originally studied in fair allocation \cite{foley1967resource}, envy-freeness has been recently adopted by preference-based approaches to fair classification \cite{balcan2018envy,ustun2019fairness,kim2019preference}, with the guiding principle that it is fair to apply different policies to different individuals or groups as long as it benefits everyone. Following this principle, we consider that the personalization of recommendations is fair only if it better accommodates individuals' preferences. In contrast, we consider unfair the failure to give users a better recommendation when one such is available for others. %For example, in a job recommender system where two users A and B seek similar roles, if A prefers the job ads of B because they give access to higher-paying opportunities, then the system is unfair according to envy-freeness. The criterion thus precludes the unfair situation where users looking for the same opportunities are given recommendations of different quality.

Envy can arise from a variety of sources, for which we provide concrete examples in our experiments (Sec.~\ref{sec:exp-envy}).

\begin{remark}
We discuss an immediate extension of envy-freeness from individuals to groups of users in App. \ref{app:homogroups}, in the special case where groups have homogeneous preferences and policies. Defining group envy-free recommendations in the general case is nontrivial and left for future work.
\end{remark}

\subsection{Compatibility of envy-freeness}
\label{sec:compatibility}
\paragraph{Optimal recommendations are envy-free\footnote{App.\ref{sec:compat} shows the difference between envy-freeness and optimality certificates.}} Let $\piu{\user} \in\argmax_{\pi} \iutil^\user(\pi)$ denote an optimal recommendation policy for $m$. Then \emph{the optimal recommender system $(\piu{\user})_{\user\in\nusers}$ is envy-free} since:
$\iutil^\user(\piu{\user}) = \max_{\pi} \iutil^\user(\pi) \geq  \iutil^\user(\piu{\otheruser})$. In contrast, achieving equal user utility in general can only be achieved by decreasing the utility of best-served users for the benefit of no one. It is also well-known that achieving parity in general requires to deviate from optimal predictions \citep{barocas2018fairness}.

\paragraph{Envy-freeness vs. item-side fairness} Envy-freeness is a user-centric notion. Towards multisided fairness \citep{burke2017multisided}, we analyze the compatibility of envy-freeness with item-side fairness criteria for rankings from \citet{singh2018fairness}, based 
on sensitive categories of items (denoted $\actSsmall_1, ..., \actSsmall_{S}$). \emph{Parity of exposure} prescribes that for each user, the exposure of an item category should be proportional to the number of items in that category. In \emph{Equity of exposure}%\footnote{\citet{singh2018fairness} use a different terminology of demographic parity (resp. disparate treatment) constraints for what we call parity (resp. equity) of exposure. Our use of ``equity'' in that context follows \citet{biega2018equity}.}
\footnote{\citet{singh2018fairness} use the terminology of demographic parity (resp. disparate treatment) for what we call parity (resp. equity) of exposure. Our use of ``equity'' follows \citet{biega2018equity}.}, 
%on the other hand, stipulates that 
the exposure of item categories should be proportional to their average relevance to the user.

The optimal policies under parity and equity of exposure constraints, denoted respectively by $(\pipar{\user})_{\user=1}^\nusers$ and $(\piee{\user})_{\user=1}^\nusers$, are %formally 
defined given user $\user$ and context $\ctx$ as:
% Let us denote by $(\pipar{\user})_{\user=1}^\nusers$ the optimal policies from the users' point of view that satisfy parity of exposure, and by  $(\piee{\user})_{\user=1}^\nusers$ the optimal policies under equity of exposure. Their formal definition given user $\user$ and context $\ctx$ is:\footnote{The original criterion \citep[][Eq. 4]{singh2018fairness} would be written in our case as $\forall s, s'\in \intint{S}, \frac{1}{\card{\actSsmall_s}}\sum_{a\in\actSsmall_s} p(a) = \frac{1}{\card{\actSsmall_{s'}}}\sum_{a\in\actSsmall_{s'}} p(a)$, which is equivalent to \eqref{eq:FEparityconstraint}. A similar remark holds for the equity constraint.}
% \begin{align}
% \text{\emph{(parity)}}  &&  %\forall \user\in\intint{\nusers}, \forall \ctx\in\ctxS:  
%     \pipar{\user}(.|\ctx) = \argmax_{\substack{p:\actS\to[0,1]\\\sum_a p(a) = 1}}  \sum_{a\in\actS} p(a)\exprew^\user(a |\ctx)\nonumber\\ 
%     &&
%         \text{u.c. } \forall s\in \intint{S},  \sum_{a\in\actSsmall_s} p(a) = \frac{\card{\actSsmall_s}}{\card{\actS}}.\label{eq:FEparityconstraint}\\
% \text{\emph{(equity)}}   && \piee{\user}(.|\ctx) = \argmax_{\substack{p:\actS\to[0,1]\\\sum_a p(a) = 1}}  \sum_{a\in\actS} p(a)\exprew^\user(a |\ctx)\nonumber \\ %\label{eq:FEequityobj}\\ 
%     &&  \text{u.c. } \forall s \in \intint{S}, \sum_{a\in\actSsmall_s} p(a) = \frac{\sum\limits_{a\in\actSsmall_s}  \exprew^\user(a |\ctx)}{\sum\limits_{a\in\actS}  \exprew^\user(a |\ctx)}\nonumber%\label{eq:FEequity}
% \end{align}
%
\begin{align}
\text{\emph{(parity)}}  &&  %\forall \user\in\intint{\nusers}, \forall \ctx\in\ctxS:  
    \pipar{\user}(.|\ctx) = \argmax_{\substack{p:\actS\to[0,1]\\\sum_a p(a) = 1}}  \sum_{a\in\actS} p(a)\exprew^\user(a |\ctx)\nonumber\\ 
    &&
        \text{u.c. } \forall s\in \intint{S},  \sum_{a\in\actSsmall_s} p(a) = \frac{\card{\actSsmall_s}}{\card{\actS}}.\label{eq:FEparityconstraint}
\end{align}
Optimal policies under equity of exposure are defined similarly\footnote{The original criterion \citep[][Eq. 4]{singh2018fairness} would be written in our case as $\forall s, s'\in \intint{S}, \frac{1}{\card{\actSsmall_s}}\sum_{a\in\actSsmall_s} p(a) = \frac{1}{\card{\actSsmall_{s'}}}\sum_{a\in\actSsmall_{s'}} p(a)$, which is equivalent to \eqref{eq:FEparityconstraint}. A similar remark holds for the equity constraint.}, but the constraints are $\forall s, \sum\limits_{a\in\actSsmall_s} p(a) = \frac{\sum\limits_{a\in\actSsmall_s}  \exprew^\user(a |\ctx)}{\sum\limits_{a\in\actS}  \exprew^\user(a |\ctx)}$.
% \begin{align}
% \text{\emph{(equity)}}   && \piee{\user}(.|\ctx) = \argmax_{\substack{p:\actS\to[0,1]\\\sum_a p(a) = 1}}  \sum_{a\in\actS} p(a)\exprew^\user(a |\ctx)\nonumber \\ %\label{eq:FEequityobj}\\ 
%     &&  \text{u.c. } \forall s \in \intint{S}, \sum_{a\in\actSsmall_s} p(a) = \frac{\sum\limits_{a\in\actSsmall_s}  \exprew^\user(a |\ctx)}{\sum\limits_{a\in\actS}  \exprew^\user(a |\ctx)}\nonumber%\label{eq:FEequity}
% \end{align}
%
We show their relation to envy-freeness:
\begin{proposition}\label{prop:envyEE} With the above notation:
\begin{itemize}
    \item the policies $(\pipar{\user})_{\user=1}^\nusers$ are envy-free, while
    \item the policies $(\piee{\user})_{\user=1}^\nusers$ are not envy-free in general.
\end{itemize}
\end{proposition}
Optimal recommendations under parity of exposure are envy-free because the parity constraint \eqref{eq:FEparityconstraint} is the same for all users. Given two users $\user$ and $\otheruser$, $\pipar{\user}$ is optimal for $\user$ under \eqref{eq:FEparityconstraint} and $\pipar{\otheruser}$ satisfies the same constraint, so we have $\iutil^{\user}(\pipar{\user}) \geq \iutil^\user(\pipar{\otheruser})$.

In contrast, the optimal recommendations under equity of exposure are, in general, not envy-free. A first reason is that less relevant item categories reduce the exposure of more relevant categories: a user who prefers item $a$ but who also likes item $b$ from another category envies a user who only liked item is $a$. Note that \emph{amortized} versions of the criterion and other variants considering constraint averages over user/contexts \citep{biega2018equity,patro2020fairrec} have similar pitfalls unless envy-freeness is explictly enforced, as in \citet{patro2020fairrec} who developed an envy-free algorithm assuming the true preferences are known. For completeness, we describe in App.\ref{sec:compat} a second reason why equity of exposure constraints create envy, and an edge case where they do not.

\subsection{Probabilistic relaxation of envy-freeness}\label{sec:relax}

Envy-freeness, as defined in Sec. \ref{sec:envy-def}, (a) compares the recommendations of a target user to those of \emph{all} other users, and (b) these comparisons must be made for \emph{all} users. In practice, as we show, this means that the sample complexity of the audit increases with the number of users, and that all users must be part of the audit. 

In practice, it is likely sufficient to relax both conditions on all users to give a guarantee for most recommendation policies and most users. Given two small probabilities $\lambda$ and $\gamma$,
%For real-world recommender systems with a large number of users, it is more realistic to provide an envy-freeness guarantee in percentage over users, rather than to compare all users pairwise. The benefit of a probabilistic criterion is to drastically reduce the sample size required to certify $\epsilon$-envy-freeness for exactly all users. 
%
the relaxed criterion we propose requires that for at least $1-\lambda$ fraction of users, the utility of users for their own policy is in the top-$\gamma\%$ of their utilities for anyone else's policy. The formal definition is given below. The fundamental observation, which we prove in Th. \ref{th:audit-proba} in Sec. \ref{sec:full-audit}, is that % for every desired $\lambda>0$ and $\gamma>0$, 
the sample complexity of the audit and the number of users impacted by the audit are now \emph{independent on the total number of users}. 
We believe that these relaxed criteria are thus likely to encourage the deployment of envy-free audits in practice.

% For large-scale recommender systems with
% many users, it is more realistic to provide an envy-freeness guarantee in percentage over users, rather than to compare all users pairwise. The benefit of a probabilistic criterion is to drastically reduce the sample size required to certify $\epsilon$-envy-freeness for exactly all users. The relaxed criterion we propose requires that for a large fraction of users, the utility of a user for their own policy is in the top-$\gamma\%$ of their utilities for anyone's policy.
% %For real-world recommender systems with large user databases, the sample complexity required by an $\epsilon$-envy-freeness certificate is not manageable because it is quadratic in the number of users. We address this tractability issue by proposing a probabilistic relaxation of $\epsilon$-envy-freeness. 

% % \begin{definition}\label{dfn:proba-envyfree}
% % Let $\epsilon,\gamma,\lambda \!\geq\!0$. Let $U_\nusers$ denote the discrete uniform distribution over $\intint{\nusers}.$
% % A recommender system is \emph{$(\epsilon,\gamma,\lambda)$-envy-free} if:     
% % % \begin{align*}
% % %     \mathbb{P}_{\user\sim U_\nusers}\bigg[\frac{\card{\{n\in \nusers:u^m(\pi^m) \geq u^m(\pi^n)\}}}{M} \geq 1-&\gamma \bigg] \\
% % %     &\geq 1 - \lambda.
% % % \end{align*}
% % \begin{align*}
% %     \mathbb{P}_{\user\sim U_\nusers}\bigg[\mathbb{P}_{\otheruser\sim U_\nusers}\big[u^m(\pi^m) + \epsilon \geq u^m(\pi^n)\big] \geq 1-\gamma \bigg]
% %     \geq 1 - \lambda.
% % \end{align*}
% % \end{definition}

\begin{definition}\label{dfn:proba-envyfree}
Let $\epsilon,\gamma,\lambda \!\geq\!0$. Let $U_\nusers$ denote the discrete uniform distribution over $\intint{\nusers}.$
A user $\user$ is \emph{$(\epsilon,\gamma)$-envious} if:     
\begin{align*}
    \mathbb{P}_{\otheruser\sim U_\nusers}\big[u^m(\pi^m) + \epsilon < u^m(\pi^n)\big] > \gamma.
\end{align*}
A recommender system is \emph{$(\efparams)$-envy-free} if at least a $(1-\lambda)$ fraction of its users are not \emph{$(\epsilon,\gamma)$-envious}.
\end{definition}

% \virginie{should we make a distinction in the name of the criterion, compared to the exact criterion? also do we keep "envy-free" for $0$-envy-free?}

% If a recommender system is $\epsilon$-envy-free, then it is $(\efparams)$-envy-free for any $\gamma,\lambda\geq 0.$ In particular, we consistently use $\epsilon$-envy-free for $(\epsilon,0,0)$-envy-free.

% As we see in Sec. \ref{sec:full-audit}, certifying this simple relaxed criterion only requires to audit envy for a constant number of users that \emph{does not depend on the total number of users \nusers}. This crucially makes the criterion tractable, even for strong envy-freeness guarantees involving small values of $\gamma,\lambda$. 

\section{Certifying envy-freeness}\label{sec:algo}

\newcommand{\envyfree}{{\tt envy-free}\xspace}
\newcommand{\notenvyfree}{{\tt not-envy-free}\xspace}
\newcommand{\envy}{{\tt envy}\xspace}
\newcommand{\noenvy}{{\tt no-envy}\xspace}

\subsection{Auditing scenario}
The envy-freeness auditor must answer the counterfactual question: ``had user $\user$ been given the recommendations of user $\otheruser$, would $\user$ get higher utility?''. The main challenge is that the answer requires to access to user preferences, which are only partially observed since users only interact with recommended items. There is thus a need for an active exploration process that recommends items which would not have been recommended otherwise. 

To make such an exploration possible, we consider the following auditing scenario: at each time step $t$, the auditor chooses to either (a) give the user a ``normal'' recommendation, or (b) explore user preferences by giving the user a recommendation from another user (see Fig. \ref{fig:audit-scenar})
. This scenario has the advantage of lightweight infrastructure requirements, since the auditor only needs to query another user's policy, rather than  implementing a full recommender system within the operational constraints of the platform. Moreover, this interface is sufficient to estimate envy because envy is defined based on the performance of other user's policies.
This type of internal audit \cite{raji2020closing} requires more access than usual external audits that focus on recommendation parity, but this is necessary to explore user preferences.

We note that the auditor must make sure that this approach follows the relevant ethical standard for randomized experiments in the context of the audited system. The auditor must also check that using other users' recommendation policies does not pose privacy problems. From now on, we assume these issues have been resolved.

\begin{figure}
    \centering
    \includegraphics[width=0.6\linewidth]{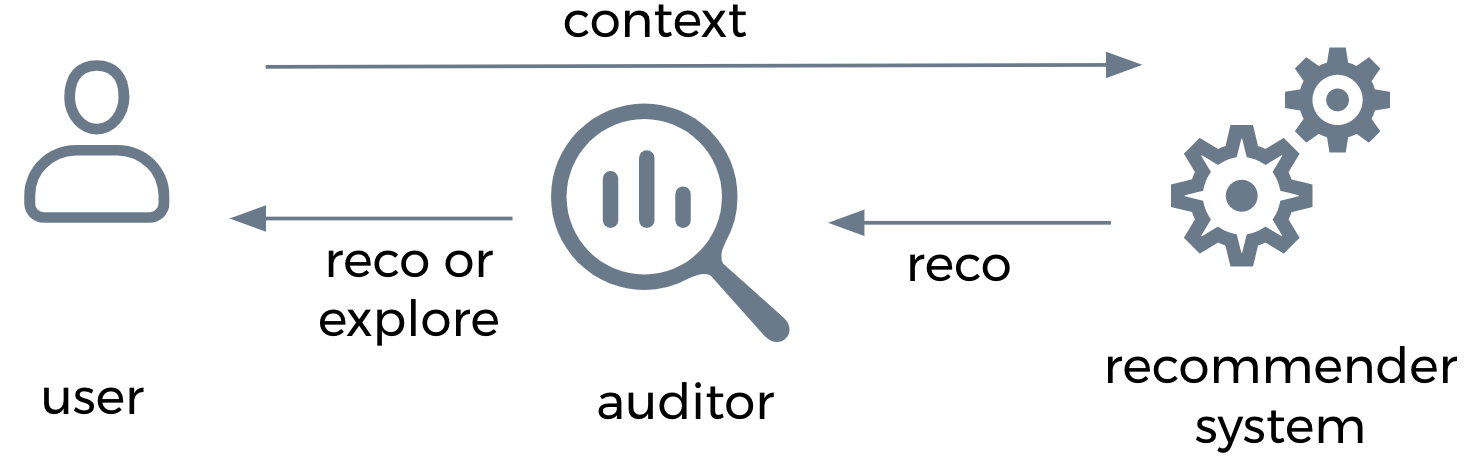}
    \caption{Auditing scenario: the auditor either shows the user their recommendation in the current rec. system, or explores by showing the recommendation given to another user.  \label{fig:audit-scenar}}
\end{figure}

\subsection{The equivalent bandit problem}
% We now turn to our auditing algorithm for envy-freeness. We first describe the algorithm which estimates envy for a single target user, then we specify its use for the certification of either the exact or probabilistic envy-freeness criteria.
We now cast the audit for envy-freeness as a new variant of pure exploration bandit problems. We first focus on auditing envy for a single target user and define the corresponding objectives, then we present our auditing algorithm. Finally we specify how to use it for the certification of either the exact or probabilistic envy-freeness criteria.

For a target user $\user$, the auditor must estimate whether $u^m(\pi^m) + \epsilon \geq u^m(\pi^n)$, for $n$ in a subset $\{n_1,...,n_K\}$ of $K$ users from $\intint{\nusers}$ (where $K$ is specified later, depending on the criterion). As we first focus on auditing envy for one target user $\user$, we drop all superscripts $m$ to simplify notation. We identify $\{n_1,...,n_K\}$ with $\intint{K}$ and rename $\big(u^m(\pi^{n_1}),..., u^m(\pi^{n_K})\big)$ as $(\mu_1,...,\mu_\narms).$ To estimate $\util_k$, we obtain samples by making recommendations using the policy $\pi^k$ and observing the reward. The remaining challenge is to choose which user $\grarm$ to sample at each time step while not deteriorating the experience of the target user too much. Index $0$ represents the target user: we use $\util_0$ for the utility of the user for their policy (i.e., $u^m(\pi^m)$). Because the audit is a special form of bandit problem, following the bandit literature, an index of a  user is called an \emph{arm}, and arm $0$ is the \emph{baseline}.

\begin{algorithm}[t]
\caption{\banditalg algorithm. $\xi_t$ (line 4) evaluates the conservative exploration constraint and is defined in \eqref{eq:xi}. Values for $\rad_k(t)$ and confidence bounds $\lcb_k$ and $\ucb_k$ are given in Lemma \ref{lem:chernoff2}.\label{alg:bandit-sync}}
\DontPrintSemicolon
 \SetKwInOut{Input}{input}\SetKwInOut{Output}{output}
 
 \Input{Confidence parameter $\delta$, conservative exploration parameter $\alpha$, envy parameter $\epsilon$}
 \Output{\envy or $\epsilon-$\noenvy}
 $S_0 \gets \intint{\narms}$    \tcp*[l]{all\! arms\! except\! $0$}

 \For{t=1, \ldots}{
  Choose $\ell_t$ from $S_{t-1}$ \tcp*[l]{e.g., \!unif. \!\!\!\!\!sample}
  \lIf{$\rad_0(t{-}1)\!>\!\!\min\limits_{\grarm \in S_{t{-}1}}\!\! \rad_\grarm(t{-}1)$ {\bf or} $\xi_t\!<\!0$
  }{
    \!$k_t \gets0$
    }
  \lElse{
   $k_t \gets \ell_t$
  }
   Observe context $x_t\sim q$, show $\act_t \sim \pi^{k_t}(.|\ctx_t)$ and observe $\rew_t\sim \nu(\act_t| \ctx_t)$ \tcp*[l]{i.e., pull arm $k_t$ and update conf. \!\!\!\!\!intervals with Lem.\!\!\!\!\! \ref{lem:chernoff2}}
   
  $S_t \gets \big\{k \in S_{t-1} : \ucb_k(t) > \lcb_0(t) + \epsilon \big\}$

\lIf{$\exists k\in S_t, \lcb_{k}(t)> \ucb_0(t)$}{\Return ~\envy}
\lIf{$S_t = \emptyset$}{
 \Return $\epsilon$-\noenvy}
}
\end{algorithm}

\paragraph{Objectives and evaluation metrics}
We present our algorithm \bandialgcomplete in the next subsection. Given $\epsilon >0$ and $\alpha\geq 0$, \banditalg returns either \envy or $\epsilon$-\noenvy and has two objectives:
\begin{enumerate}[leftmargin=*]
    \item Correctness: if \banditalg returns \envy, then $\exists \grarm, \util_\grarm > \util_0$. If \banditalg returns $\epsilon$-\noenvy then %$\forall \grarm\in\intint{\narms}, 
    $\max\limits_{ \grarm\in\intint{\narms}}\util_\grarm\leq \util_0+\epsilon$.
    \item Recommendation performance: during the audit, \banditalg must maintain a fraction $1{-}\alpha$ of the baseline performance. Denoting by $\grarm_s\in\{0, \ldots, \narms\}$ the arm (group index) chosen %by \banditalg
    at round $s$, 
    this requirement is formalized as a conservative exploration constraint \cite{wu2016conservative}:\begin{align}\label{eq:c-constraint}
    \forall t, \frac{1}{t}\sum_{s=1}^t \util_{\grarm_s} \geq (1 - \alpha) \util_0 \,.
    %~~\iff~~ \sum_{\grarm=1}^\narms \util_{\grarm} N_\grarm(t) \geq ((1-\alpha)t - N_0(t)) \util_0
\end{align} 
\end{enumerate}

We focus on the \emph{fixed confidence} setting, where given a confidence parameter $\delta\in(0,1)$ the algorithm provably satisfies both objectives with probability $1-\delta$. In addition, there are two 
criteria to assess an online auditing algorithm:
\begin{enumerate}[leftmargin=*]
\item Duration of the audit: the number of time-steps before the algorithm stops.
\item Cost of the audit: the cumulative loss of rewards incurred. Denoting the duration by $\tau$, the cost is $ \tau\util_0 - \sum_{s=1}^\tau \util_{k_s}$.

It is possible that the cost is negative when there is envy. In that case, the audit increased recommendation performance by finding better recommendations for the group.
\end{enumerate}
We note the asymmetry in the return statements of the algorithm: \envy does not depend on $\epsilon$. This asymmetry is necessary to obtain finite worst-case bounds on the duration and the cost of audit, as we see in Theorem \ref{thm:all}. 

Our setting had not yet been addressed by the pure exploration bandit literature, which mainly studies the identification of ($\epsilon$-)optimal arms \citep{audibert2010best}. Auditing for envy-freeness requires proper strategies in order to efficiently estimate the arm performances compared to the unknown baseline. Additionally, by making the cost of the audit a primary evaluation criterion, we also bring the principle of conservative exploration to the pure exploration setting, while it had only been studied in regret minimization \citep{wu2016conservative}. In our setting, conservative constraints involve nontrivial trade-offs between the duration and cost of the audit. We now present the algorithm, and then the theoretical guarantees for the objectives and evaluation measures.

% Our setting had not yet been addressed by the pure exploration bandit literature, which mainly studies the identification of ($\epsilon$-)optimal arms \citep{}. It is in fact closer to threshold bandits \citep{locatelli2016optimal} where the goal is to find arms higher than a known fixed threshold, but the challenge in our case is that the value of the baseline is unknown. Auditing for envy-freeness thus requires appropriate strategies in order to efficiently estimate the performance of some arms compared to an unknown baseline. Additionally, we make the cost of the audit a primary evaluation criterion, which pushes towards the principle of conservative exploration. These constraints had only been studied in regret minimization \citep{wu2016conservative,garcelon2020conservative}, and are new to the pure exploration setting, where they involve nontrivial trade-offs between duration and cost of the audit.

\subsection{The \banditalg algorithm}

\banditalg is described in Alg.~\ref{alg:bandit-sync}.
%At each iteration $t$, the algorithm pulls an arm $\grarm_t$ and observes a reward $\rew_t\in[0,1]$. 
It maintains confidence intervals on arm performances $(\util_\grarm)_{\grarm=0}^\narms$. Given the confidence parameter $\delta$, the lower and upper  bounds on $\util_\grarm$ at time step $t$, denoted by $\lcb_\grarm(t)$ and $\ucb_\grarm(t)$,
%Given an arm $\grarm\in\compgrext$, the lower and upper confidence bounds on $\util_\grarm$ are denoted by $\lcb_\grarm(t)$ and $\ucb_\grarm(t)$ where the half-length of the interval is denoted by $\beta_k(t)$, i.e., $\beta_k(t) = \ucb_\grarm(t) - \lcb_k(t)$. Given a confidence parameter $\delta$ that is the confidence parameter of the test, $\lcb_\grarm(t)$ and $\ucb_\grarm(t)$ 
are chosen so that with probability at least $1-\delta$, we have
%that the confidence intervals are valid for all arms and all timesteps with probability $\geq 1-\delta$.
$\forall k,t, \util_\grarm\in[\lcb_k(t), \ucb_k(t)]$.  In the algorithm, $\beta_k(t) = (\ucb_\grarm(t) - \lcb_k(t))/2$. As \citet{jamieson2014lil}, we use anytime bounds inspired by the law of the iterated logarithm. These are given in Lem.~\ref{lem:chernoff2} in App. \ref{sec:proofs}.
% Other bounds using, e.g.,  empirical Bernstein inequalities \cite{maurer2009empirical}, or more simply Hoeffding's inequality \cite{hoeffding1963probability}, could also be used.

%The confidence intervals are computed based on the observed reward seen when playing the arms, and they become tighter when the arms are played more. 
%Precise formulas are given in Lemma \ref{lem:chernoff2}.

\banditalg maintains an active set $S_t$ of all arms in $\intint{\narms}$ (i.e., excluding the baseline) whose performance are not confidently less than $\util_0+\epsilon$. It is initialized to $S_0 = \intint{\narms}$ (line 1). At each round $t$, the algorithm selects an arm $\ell_t\in S_t$ (line 3). Then, depending on the state of the conservative exploration constraint (described later), the algorithm pulls $k_t$, which is either $\ell_t$ or the baseline (lines 4-6). After observing the reward $\rew_t$, the confidence interval of $\util_{\ell_t}$ is updated, and all active arms that are confidently worse than the baseline plus $\epsilon$ are de-activated
%, i.e., we discard all active arms $\grarm$ such that $\ucb_k(t) \leq \lcb_0(t) + \epsilon$ 
(line 7). The algorithm returns \envy if an arm $\grarm$ is confidently better than the baseline (line 8)%, i.e., $\lcb_\grarm(t) > \ucb_0(t)$ (line 8)
, returns $\epsilon$-\noenvy if there are no more active arms, (line 9) or continues if neither of these conditions are met.

\paragraph{Conservative exploration}
To deal with the conservative exploration constraint \eqref{eq:c-constraint}, we follow \cite{garcelon2020conservative}.
Denoting $\activepullS_t = \{s\leq t:\grarm_s\neq 0\}$ the time steps at which the baseline was not pulled, we maintain a confidence interval  such that with probability $\geq1-\delta$, we have $\forall t>0, \big\vert\sum_{s\in\activepullS_t}(\util_{\grarm_s} - \rew_s) \big\vert \leq  \Phi(t)$. The formula for $\Phi$ is given in Lem.~\ref{lem:bound-constraint} in App. \ref{sec:proofs}. This confidence interval is used to estimate whether the conservative constraint \eqref{eq:c-constraint} is met at round $t$ as follows. First, let us denote by $N_\grarm(t)$ the number of times arm $\grarm$ has been pulled until $t$, and notice that \eqref{eq:c-constraint} is equivalent to $\sum_{s\in\activepullS_t} \util_{\grarm_s} - ((1-\alpha)t - N_0(t)) \util_0 \geq 0$.  After choosing $\ell_t$ (line 3), we use the lower bound on $\sum_{s \in \activepullS_{t}} \util_{\grarm_s}$ and the upper bound for $\util_0$ to obtain a conservative estimate of \eqref{eq:c-constraint}. Using $\tau=t-1$, this leads to:
\begin{equation}\label{eq:xi}
\resizebox{.9\linewidth}{!}{$\displaystyle
    \xi_t = \sum\limits_{s\in \activepullS_{\tau}} \!\rew_s - \Phi(t) + \lcb_{\ell_t}\!\!(\tau)+ (N_0(\tau) - (   1-\alpha) t) \ucb_0(\tau)\,.
$}
\end{equation}

Then, as long as the confidence intervals hold, pulling $\ell_t$ does not break the constraint~\eqref{eq:c-constraint} if $\xi_t \geq 0$. The algorithm
thus pulls the baseline arm when $\xi_t<0$. To simplify the theoretical analysis, \banditalg also pulls the baseline if it does not have the tightest confidence interval (lines 4-6).

\subsection{Analysis}

The main theoretical result of the paper is the following:
\nico{check the bound for $\delta$ and the duration/cost}
% \begin{theorem}\label{thm:all}
% Let $\epsilon\in(0,1]$, $\alpha\in(0,1], \delta\in(0,\frac{1}{2})$ and $\eta_\grarm = \max(\util_\grarm-\util_0, \util_0 + \epsilon - \util_\grarm)$. Using $\lcb,\ucb$ and $\Phi$ given in Lemmas~\ref{lem:chernoff2} and \ref{lem:bound-constraint} (App.~\ref{sec:proofs}), \banditalg achieves the following guarantees with probability at least $1-\delta$:
% \begin{itemize}[leftmargin=*]
%     \item \banditalg is correct and satisfies the conservative constraint on the recommendation performance \eqref{eq:c-constraint}.
%     \item The duration is in %$O\big(\frac{1}{\card{\group}}\sum\limits_{k=1}^K\frac{1}{\alpha\util_0\eta_k+\eta_k^2}\log\big(\frac{-\narms\log(\alpha\mu_0\eta_k^2)}{\delta}\big)\big)$
%     $\displaystyle O\bigg(\frac{1}{\card{\group}}\sum_{k=1}^K\frac{\log\big(\frac{-\narms\log(\alpha\mu_0\eta_k^2/\narms)}{\delta}\big)}{\alpha\util_0\eta_k+\eta_k^2}\bigg)$.
%     \item The cost is in $ O\bigg(\mathlarger{\mathlarger{\sum}}\limits_{k:\util_k<\util_0}\!\!\frac{(\util_0 - \util_k)\log\big(\frac{-\narms\log(\delta\card{\group}\eta_k^2/\narms)}{\delta}\big)}{\card{\group}\eta_k^2}\bigg)$.
% \end{itemize}
% \end{theorem}
\begin{theorem}\label{thm:all}
Let $\epsilon\in(0,1]$, $\alpha\in(0,1], \delta\in(0,\frac{1}{2})$ and $
    \eta_\grarm = \max(\util_\grarm-\util_0, \util_0 + \epsilon - \util_\grarm) \text{~and~} h_\grarm = \max(1,\frac{1}{\eta_\grarm}).$
Using $\lcb,\ucb$ and $\Phi$ given in Lemmas~\ref{lem:chernoff2} and \ref{lem:bound-constraint} (App.~\ref{sec:proofs}), \banditalg achieves the following guarantees with probability $\geq1-\delta$:
\begin{itemize}[leftmargin=*]
    \item \banditalg is correct and satisfies the conservative constraint on the recommendation performance \eqref{eq:c-constraint}.
    \item The duration is in
    $\displaystyle O\bigg(\sum_{k=1}^K\frac{h_\grarm \log\big(\frac{\narms\log(\nicefrac{\narms h_\grarm}{\delta \eta_\grarm)}}{\delta}\big)}{\min(\alpha\util_0,\eta_k)}\bigg)$.
    \item The cost is in $ O\bigg(\mathlarger{\mathlarger{\sum}}\limits_{k:\util_k<\util_0}\!\!\frac{(\util_0 - \util_k) h_\grarm}{\eta_k}\log\big(\frac{\narms\log(\nicefrac{\narms h_\grarm}{\delta \eta_\grarm)})}{\delta}\big)\bigg)$.
\end{itemize}
\end{theorem}
The important problem-dependent quantity $\eta_\grarm$ is the gap between the baseline and other arms $\grarm$. It is asymmetric depending on whether the arm is better than the baseline $(\util_\grarm-\util_0)$ or the converse  ($\util_0 - \util_\grarm + \epsilon$) because the stopping condition for \envy does not depend on $\epsilon$. This leads to a worst case that only depends on $\epsilon$, since $\eta_\grarm = \max(\util_\grarm-\util_0, \util_0 - \util_\grarm + \epsilon) \geq \frac{\epsilon}{2}$, while if the condition was symmetric, we would have possibly unbounded duration when $\util_k = \util_0 + \epsilon$ for some $\grarm\neq 0$. Overall, ignoring log terms, we conclude that when $\alpha\util_0$ is large, the duration is of order $\sum_k \frac{1}{\eta_k^2}$ and the cost is of order $\sum_k \frac{1}{\eta_k}$. This becomes $\sum_k\frac{1}{\alpha\util_0 \eta_k}$ and $\sum_k \frac{1}{\eta_k}$ when $\alpha\util_0$ is small compared to $\eta_k$. This means that the conservative constraint has an impact mostly when it is %relatively 
strict. It also means that when either $\alpha\util_0 \ll \eta_k$ or $\eta_k^2 \ll \eta_k$ the cost can be small even when the duration is fairly high. %In all cases, it is critical to note that both the cost and duration, which are per-user quantities, decrease linearly with the group size $\card{\group}$. Thus assessing group envy-freeness becomes more sample efficient (per user) as groups become larger.

\subsection{Full audit} \label{sec:full-audit}
\paragraph{Exact criterion} To audit for envy-freeness on the full system, we apply \banditalg to all $M$ users simultaneously and with $K=M$, meaning that the set of arms corresponds to all the users' policies. By the union bound, using $\delta' = \frac{\delta}{\nusers}$ instead of $\delta$ in \banditalg's confidence intervals, the guarantees of Theorem~\ref{thm:all} hold simultaneously for all users.

For recommender systems with large user databases, the duration of \banditalg thus becomes less manageable as $\nusers$ increases. We show how to use \banditalg to certify the probabilistic criterion with guarantees that do not depend on $\nusers$.

\paragraph{Probabilistic criterion} The \auditalg algorithm for auditing the full recommender system is described in Alg. \ref{alg:full-audit}. \auditalg samples a subset of users and a subset of arms for each sampled user. Then it applies \banditalg to each user simultaneously with their sampled arms. It stops either upon finding an envious user, or when all sampled users are certified with $\epsilon$-no envy. Again there is a necessary asymmetry in the return statements of \auditalg to obtain finite worst-case bounds whether or not the system is envy-free.

The number of target users $\nsampled$ and arms $\narms$ in Alg.~\ref{alg:full-audit} are chosen so that $\epsilon$-envy-freeness \emph{w.r.t.} the sampled users and arms translates into $(\efparams)$-envy-freeness. Combining these random approximation guarantees with Th.~\ref{thm:all}, we get:
 
% The number of arms $\narms$ in Alg.~\ref{alg:full-audit} is chosen so that ``not $\epsilon$-envious'' \emph{w.r.t.} the $\narms$ arms translates to ``not $(\epsilon,\gamma)$-envious'' with high probability, and the number of target users is set to $\nsampled=\ceil{\frac{\log(3/\delta)}{\lambda}}$ to also guarantee that $\epsilon$-envy-freeness \emph{w.r.t.} $\nsampled$ users and $\narms$ arms implies $(\efparams)$-envy-freeness. We combine these random approximation guarantees with Th.~\ref{thm:all} applied simultaneously to all $\nsampled$ sample users: 

% The number of arms is set to $\narms=\ceil{\frac{\log(3\nsampled/\delta)}{\log(1/(1-\gamma))}}$, so that if a user $\user$ is not $\epsilon$-envious \emph{w.r.t.} the $\narms$ arms,
% then $\user$ is not $(\epsilon,\gamma)$-envious with high probability. The number of target users is set to $\nsampled=\ceil{\frac{\log(3/\delta)}{\lambda}}$ to also guarantee that $\epsilon$-envy-freeness \emph{w.r.t.} all these users implies $(\efparams)$-envy-freeness. By a union bound, we combine these random approximation guarantees with Th.~\ref{thm:all} applied simultaneously to all $\nsampled$ sampled users: 
\begin{theorem} \label{th:audit-proba}
Let $\nsampled = \ceil{\frac{\log(3/\delta)}{\lambda}}$ and $\narms=\ceil{\frac{\log(3\nsampled/\delta)}{\log(1/(1-\gamma))}}$. With probability $1-\delta$, \auditalg is correct, it satisfies the conservative constraint \eqref{eq:c-constraint} for all $\nsampled$ target users, and the bounds on duration and cost from Th.~\ref{thm:all} (using $\frac{\delta}{3\tilde{M}}$ instead of $\delta$) are simultaneously valid. 
\end{theorem}

Importantly, in contrast to naively using \banditalg to compare all users against all, the audit for the probabilistic relaxation of envy-freeness only requires to query a constant number of users and policies that \emph{does not depend on the total number of users $\nusers$}. Therefore, the bounds on duration and cost are also independent of $\nusers$, which is a drastic improvement.

\begin{algorithm}[t]
\caption{\auditalg algorithm. The algorithm either outputs a probabilistic certificate of $(\efparams)$-envy-freeness, or evidence of envy. \label{alg:full-audit}}
\DontPrintSemicolon
 \SetKwInOut{Input}{input}\SetKwInOut{Output}{output}
 
 \Input{Confidence parameter $\delta$, conservative exploration parameter $\alpha$, envy parameters $(\efparams)$}
 \Output{$(\efparams)$-\envyfree or \notenvyfree}
 Draw a sample $\tilde{S}$ of $\nsampled=\ceil{\frac{\log(3/\delta)}{\lambda}}$ users from $\intint{\nusers}$\;
 \For{each user $m \in \tilde{S}$ in parallel}{
    Sample $\narms=\ceil{\frac{\log(3\nsampled/\delta)}{\log(1/(1-\gamma))}}$ arms from $\intint{\nusers}\setminus{\{m\}}$\;
    Run OCEF\big($\frac{\delta}{3\nsampled},\alpha,\epsilon$\big) for user $m$ with the $\narms$ arms\;
    \lIf{\banditalg outputs \envy}{\Return ~\notenvyfree}
 }
\Return $(\efparams)$-\envyfree
\end{algorithm}

\section{Experiments} \label{sec:exps}

% \virgi{\begin{enumerate}
%     \item Evidence of envy: either mispecification or measurement bias (maybe mispecification is better)
%     \item Audit algorithm: scaling wrt alpha of the cost, but also the parameters of the probabilistic criterion (instead of group size?)
% \end{enumerate}}

We present experiments describing sources of envy (Sec. \ref{sec:exp-envy}) and evaluating the auditing algorithm \banditalg on two recommendation tasks (Sec. \ref{sec:exp-ocef}). 

We create a music recommendation task based on the Last.fm dataset from \citet{Cantador:RecSys2011}, which contains the music listening histories of $1.9$k users. We select the $2500$ items most listened to, and simulate ground truth user preferences by filling in missing entries with a popular matrix completion algorithm for implicit feedback data\footnote{Using the Python library Implicit: \url{https://github.com/benfred/implicit} (MIT License).}. We also address movie recommendation with the MovieLens-1M dataset \cite{harper2015movielens}, which contains ratings of movies by real users, and from which we extract the top $2000$ users and $2500$ items with the most ratings. We binarize ratings by setting those $<3$ to zero, and as for Last.fm we complete the matrix to generate ground truth preferences.

% We simulate a movie recommendation task using the MovieLens-1M dataset %\footnote{\url{https://grouplens.org/datasets/movielens/1m/}}
% \cite{harper2015movielens}, which contains ratings of movies by real users. We extract a $2000\times2000$ user-item matrix, keeping users and items with the most ratings, and fill in missing entries using a popular matrix completion algorithm\footnote{Using \url{http://github.com/gbolmier/funk-svd}.}. %the FancyImpute package\footnote{\url{https://pypi.org/project/fancyimpute/}}.
% This completed matrix serves as ground truth user preferences. 

% We also consider music recommendation with the Last.fm dataset from \cite{Cantador:RecSys2011} which contains the music listening histories of $1.9$k users. We select the top $2,500$ items most listened to, and as for MovieLens we obtain ground truth preferences by matrix completion. \virgi{a word on implicit? or use implicit for everything?}

For both recommendation tasks, the simulated recommender system estimates relevance scores using low-rank matrix completion \cite{bell1995information} on a training sample of $70\%$ of the ground truth preferences, where the rated / played items are sampled uniformly at random. Recommendations are given by a fixed-temperature  
\emph{softmax} policy over the predicted scores. We generate binary rewards using a Bernoulli distribution with expectation given by our ground truth preferences.

\subsection{Sources of envy}\label{sec:exp-envy}
We consider two measures of the degree of envy. Denoting $\degenvy^m = \max(\max\limits_{n\in \intint{\nusers}}\iutil^m(\pi^n)- \iutil^m(\pi^m), 0)$, these are:
\begin{itemize}[leftmargin=*]
    \item the average envy experienced by users: $\frac{1}{\nusers}\sum\limits_{m\in \intint{\nusers}}\degenvy^m $,
    \item the proportion of $\epsilon$-envious users: $\frac{1}{\nusers}\sum\limits_{m\in \intint{\nusers}}\indic{\degenvy^m > \epsilon}$.
\end{itemize}
\subsubsection{Envy from model mispecification }
\begin{figure}
    \centering
    \includegraphics[width=\linewidth]{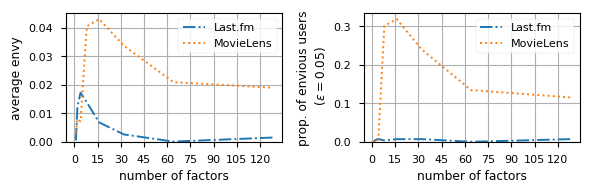}
    \caption{Envy from model mispecification on MovieLens and Lastfm: envy is high when the latent factor model is mispecified, but it decreases as the number of factors increases.\label{fig:mispeargmaxageocc}}
\end{figure}

We demonstrate that envy arises from a standard recommendation model when the modeling assumptions are too strong.  We vary the number of latent factors of the matrix completion model and evaluate a softmax policy with inverse temperature set to $5$. In Fig. \ref{fig:mispeargmaxageocc}, with one latent factor we observe no envy. This is because all users receive the same recommendations since matrix completion is then equivalent to a popularity-based recommender system. With enough latent
factors, preferences are properly captured by the model and
the degree of envy decreases. For intermediate number of latent factors, envy is visible. 

\subsubsection{Envy from equal user utility} We show that in contrast to envy-freeness, enforcing equal user utility (EUU) degrades user satisfaction and creates envy between users. We compute optimal EUU policies and unconstrained optimal policies (OPT) on the ground truth preferences of Last.fm and MovieLens. Our results in Table \ref{tab:euu-bad} confirm the pitfalls of EUU, while illustrating that OPT policies are always envy-free. 

We discuss more sources of envy and provide the details of these computations in App. \ref{app:sources}.

\begin{table}[]
\footnotesize
\centering
\begin{tabular}{l|l|l|l|l|}
\cline{2-5}
\multirow{2}{*}{}                   & \multicolumn{2}{l|}{Last.fm} & \multicolumn{2}{l|}{MovieLens} \\ \cline{2-5} 
                                    & EUU            & OPT         & EUU             & OPT          \\ \hline
\multicolumn{1}{|l|}{Total utility} & 1552          & \textbf{1726}         & 1671           & \textbf{1761}          \\ \hline
\multicolumn{1}{|l|}{Average envy}          & 0.10            & \textbf{0}           & 0.04            & \textbf{0}            \\ \hline
\multicolumn{1}{|l|}{Prop. 0.05-envious}          & 0.61            & \textbf{0}           & 0.13             & \textbf{0}            \\ \hline
\end{tabular}
\caption{Optimal policies with equal user utility penalty (EUU) vs. Unconstrained optimal policies (OPT), computed on ground truth preferences: EUU deteriorates total utility and creates envy between users.}
\label{tab:euu-bad}
\end{table}

% \begin{table}[]
% \centering
% \begin{tabular}{|l|l|l|l|l|}
% \hline
%               & \multicolumn{2}{l|}{Last.fm} & \multicolumn{2}{l|}{MovieLens} \\ \hline
%               & EUU            & OPT         & EUU             & OPT          \\ \hline
% Total utility & small          & big         & small           & big          \\ \hline
% Envy          & big            & 0           & big             & 0            \\ \hline
% \end{tabular}
% \caption{Optimal policy under equal user utility constraint: deteriorates total utility and creates envy between user. }
% \label{tab:euu-bad}
% \end{table}

\subsection{Evaluation of the auditing algorithm}\label{sec:exp-ocef}

Our goal is now to answer for \banditalg: in practice, what is the interplay between the required sample size per user, the cost of exploration and the conservative exploration parameter? %\textbf{(2)} In practice, how does \banditalg fare\virgi{compared to other pure exploration algorithms?} when evaluated on real-world datasets? 

\subsubsection{Bandit experiments}

\begin{figure}[t]
    \centering
    \includegraphics[width=0.95\linewidth,height=3.6cm]{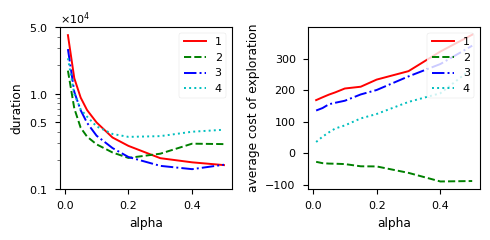}
    \caption{Effect of the conservative exploration parameter $\alpha$ on the duration and cost of auditing on Bandit experiments.\label{fig:banditexp}}
\end{figure}

\begin{figure}[t]
    \centering
    \includegraphics[width=0.95\linewidth,height=3.6cm]{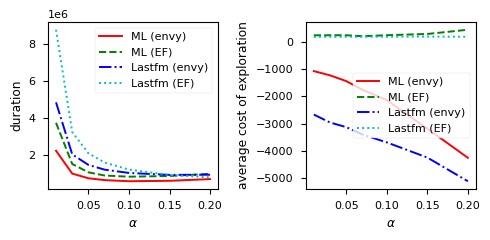}
    \caption{Scaling w.r.t. $\alpha$ on MovieLens (ML) and Last.fm, for recommender systems that are either envy-free (EF) or with envy. There are $41$ target users and $75$ arms.}
    \label{fig:ml-lasfm-ocef}
\end{figure}
% \begin{figure}[t]
%     \centering
%     \includegraphics[width=\linewidth,height=3.4cm]{figs/TODO}
%     \caption{[PLACEHOLDER FIGURE] maybe an image or table where OCEF is compared to some baseline algorithm}
% \end{figure}

We first study the trade-off between duration and cost of the audit on 4 bandit problems with Bernoulli rewards and 10 arms. %These experiments aim at studying the trade-off between duration and cost of exploration.
In Problem 1, the baseline is the best arm and all other arms are equally bad. In Prob. 2, arm $1$ is best and all other arms are as bad as the baseline. In Prob.3 the baseline is best and the means of arms from best to worst decrease rapidly. Prob. 4 uses the same means as Prob. 3,  but the means of the baseline and arm $1$ are swapped, making the baseline second-to-best. We set $\delta=\epsilon=0.05$ and report results averaged over $100$ trials. The details of the bandit configurations are given in Appendix \ref{sec:banditexps}.

Figure~\ref{fig:banditexp} plots the duration and the cost of exploration \eqref{eq:costexp} as a function of the conservative constraint parameter $\alpha$ (smaller $\alpha$ means more conservative). 
%The curves show an effect of $\alpha$ that is more subtle than what the worst-case bounds of the previous section suggest. 
The curves show that for Problems 2, 3, and 4, duration is minimal for a non-trivial $\alpha$.
%a phenomenon that was not captured by the worst case analysis (Th. \ref{th:banditsample}) which suggested that duration only increases as $\alpha$ decreases.
This is because when $\alpha$ is large, all arms are pulled as much as the baseline, so their confidence intervals are similar. When $\alpha$ decreases, the baseline is pulled more, which reduces the length of the relevant confidence intervals $\rad_0(t) + \rad_k(t)$ for \emph{all} arms $k$. This, in turn, shortens the audit because non-baseline arms are more rapidly discarded or declared better. When $\alpha$ becomes too small, however, the additional pulls of the baseline have no effect on $\rad_0(t) + \rad_k(t)$ because it is dominated by $\rad_k(t)$, so the duration only increases. This subtle phenomenon is not captured by our analysis (Th. \ref{thm:all}), because the ratios $\rad_0(t)/\rad_k(t)$ are difficult to track formally.

The sign of the cost of exploration depends on whether there is envy. In Prob. 2 where the baseline has the worst performance, exploration is beneficial to the user and so the cost is negative. On all other instances however, the cost is positive. The cost of exploration is closest to 0 when $\alpha$ becomes small because then  $\rad_0(t) + \rad_k(t)$ is the smallest possible for a given number of pulls of $k$. For instance, in Prob. 4, the cost is close to $0$ when $\alpha$ is very small and increases with $\alpha$. It is the case where the baseline is not the best arm but is close to it, and there are many bad arms. When the algorithm is very conservative, bad arms are discarded rapidly thanks to the good estimation of the baseline performance. In this ``low-cost'' regime however, the audit is significantly longer.

Appendix \ref{sec:banditexps} contains additional results when varying the number of arms and the confidence parameter $\delta$.

\subsubsection{MovieLens and Last.fm experiments}
We now evaluate the certification of the (absence of) envy of recommendation policies on MovieLens (ML) and Last.fm. We consider two recommendation policies which are softmax functions over predicted relevance scores with inverse temperature set to either $5$ or $10$. These scores were obtained by matrix completion with $48$ latent factors. On both datasets, with inverse temperature equal to $5$, the softmax recommender system is envy-free, whereas there is envy when it is set to $10$. We use \auditalg with \banditalg to certify the probabilistic criterion. The envy parameters are set to $\epsilon=\delta=0.05$ and $\lambda=\gamma=0.1$, therefore we have $\nsampled=41$ target users and $\narms=75$ arms, independently on the number of users in each dataset.

The results of applying OCEF on each dataset (ML or Last.fm) with each policy (envy-free or with envy) are shown in Fig.~\ref{fig:ml-lasfm-ocef}. For the $(\efparams)$-envy-free policies, results are averaged over $20$ trials and over all the non-$(\epsilon,\gamma)$-envious users, whereas when there is envy, results are averaged over the target users who are $\epsilon$-envious. We observe clear tendencies similar to those of the previous section, although the exact sweet spots in terms of $\alpha$ depends on the specific configuration. In particular, on envy-free configurations, the cost of the audit is positive and grows when relaxing the conservative constraint, while it is negative and decreasing with $\alpha$ when there is envy.  More details are provided in App. \ref{sec:movielensexps}.

%the utilities $\util^\gr_\grarm$ are fairly close (standard deviation $0.1$) to the group utilities $\util^\gr$ which are around $0.8$ on average. The settings are thus similar to Prob. 1, even though the means of arms are closer here and there are groups. 

\section{Conclusion}
We proposed the audit of recommender systems for user-side fairness with the criterion of envy-freeness. The auditing problem requires an explicit exploration of user preferences, which leads to a formulation as a bandit problem with conservative constraints. We presented an algorithm for this problem and analyzed its performance experimentally. % In future work, we plan to extend envy-freeness to heterogeneous groups of users, in order to generalize existing definitions of preference-based fairness to personalized predictions. 

\clearpage

\section*{Acknowledgments}
We would like to thank Jérôme Lang, Levent Sagun and the anonymous reviewers for their constructive comments on earlier versions of this paper.
\bibliography{references}

\clearpage
\appendix
\section{(In-)Compatibility of envy-freeness}\label{sec:compat}
\subsection{Envy-freeness vs. optimality certificates}
We showed in Section \ref{sec:compatibility} that envy-freeness is compatible with optimal predictions. To understand the differences between a certificate of envy-freeness and a certificate of optimality, let us denote by $\Pi^{*} = \{\pi :\exists \iutil \text{ satisfying \eqref{eq:userutil} }, \pi \in\argmax_{\pi'} \iutil(\pi')\}$ the set of potentially optimal policies. If the set of users policies approximately covers the set of potentially optimal policies $\Pi^*$, then an envy-free system is also optimal. Formally, let $\dpol(\pi, \pi')$  such that $| \iutil(\pi) - \iutil(\pi')| \leq \dpol(\pi, \pi')$. It is easy to see that if $\max\limits_{\pi \in \Pi^*} \min\limits_{\user \in \nusers} \dpol(\pi, \pi^\user) \leq \tilde{\epsilon}$, then $\epsilon$-envy-freeness implies $\epsilon+\tilde{\epsilon}$-optimality.

In practice, the space of optimal policies is much larger than the number of users (for instance, there are $\card{\actS}^{\card{\ctxS}}$ optimal policies in our setting), so that auditing for envy is tractable in cases where auditing for optimality is not.

\subsection{Envy-freeness vs. equity of exposure} We remind the definition of optimal policies with equity of exposure constraints from Section \ref{sec:compatibility}:
\begin{align}
\text{\emph{(equity)}}   && \piee{\user}(.|\ctx) = \argmax_{\substack{p:\actS\to[0,1]\\\sum_a p(a) = 1}}  \sum_{a\in\actS} p(a)\exprew^\user(a |\ctx)\nonumber \\ %\label{eq:FEequityobj}\\ 
    &&  \text{u.c. } \forall s \in \intint{S}, \sum_{a\in\actSsmall_s} p(a) = \frac{\sum\limits_{a\in\actSsmall_s}  \exprew^\user(a |\ctx)}{\sum\limits_{a\in\actS}  \exprew^\user(a |\ctx)}\nonumber%\label{eq:FEequity}
\end{align}

The constraints should be ignored when $\sum\limits_{a\in\actS}  \exprew^\user(a |\ctx) = 0.$

Following Proposition \ref{prop:envyEE} from Section \ref{sec:compatibility}, we describe here a second source of envy when using optimal policies with equity of exposure constraints. By the linearity of the optimization problem for $\piee{\user}$, the policy assigns to the best item in a category the exposure of the entire category. It implies that categories with high average relevance have more exposure than categories with few but highly relevant items. Table~\ref{tab:FEequitycounterexample} gives an example with two users and two categories of items where both users envy each other with the optimal recommendations under equity of exposure constraints.

\begin{table}[t]
    \centering
    \small
    \begin{tabular}{cc|cccc||cc}
    \toprule
        && \multicolumn{2}{c}{item cat. 1} & \multicolumn{2}{c||}{item cat. 2} &\multicolumn{2}{c}{utilities} \\
        \midrule
        (item idx) && 1 & 2 & 3 & 4 & $\iutil^1$ & $\iutil^2$ \\ \midrule
        \multirow{2}{*}{(rewards)}&$\exprew^1$ &  1 & 0 & 0.8 & 0.7& &\\
        &$\exprew^2$ & 0.8 & 0.7 & 1 & 0& & \\ \midrule
        \multirow{2}{*}{(policies)}&$\piee{1}$ & 0.4 & 0 & $0.6$ & 0 & {\bf 0.88} & {\bf 0.92}\\
        &$\piee{2}$ & $0.6$ & 0 & $0.4$ & 0 & {\bf 0.92} & {\bf 0.88}\\\bottomrule
    \end{tabular}
    \caption{Example where the optimal recommendations under item-side equity of exposure constraints are not user-side fair because both users envy each other. There are 4 items, 2 item categories and 2 users. User $1$ envies user 2 since $\iutil^1(\piee{2}) > \iutil^1(\piee{1})$. Also, $\iutil^2(\piee{1}) > \iutil^2(\piee{2})$.}
    \label{tab:FEequitycounterexample}
\end{table}

In some degenerate cases though, equity of exposure policies are envy-free.

\begin{lemma}
If for all contexts $\ctx \in \ctxS$, each user $\user \in \intint{\nusers}$ only likes a single item category $\actS_{s_m}$, i.e. $\forall a \in \actS \setminus \actS_{s_\user}, \rho^\user(\act|\ctx)= 0$, then the policies $(\piee{\user})_{\user=1}^\nusers$ are envy-free.
\end{lemma}

\begin{proof}
We set contexts $\ctx$ aside to simplify notation, but the generalization is straightforward.

We actually prove a stronger result than the lemma: if each user $\user$ only likes a single item, then $(\piee{\user})_{\user=1}^\nusers = (\pi^{\user,*})_{\user=1}^\nusers$, where $\pi^{\user,*}$ is the optimal unconstrained policy for $\user$.

Let  $\act_s^\user = \argmax_{\act \in \actS_s} \rho^\user(a)$ be the favorite item in category $\actS_s$ for user $\user$, then the optimal equity of exposure constrained policies has the following analytical expression: 
$$\forall s \in S, \forall \act \in \actS_s, \quad \piee{\user}(\act) = \indic{a = a^\user_s} \frac{\sum\limits_{a\in\actSsmall_s}  \exprew^\user(a')}{\sum\limits_{a'\in\actS}  \exprew^\user(a' )},$$
and we thus have:
\begin{align}\label{eq:ee-util}
    u^\user(\piee{\user}) = \sum_{s \in \intint{S}} \rho^\user(a_s^\user)\frac{\sum\limits_{a\in\actSsmall_s}  \exprew^\user(a)}{\sum\limits_{a\in\actS}  \exprew^\user(a)}.
\end{align}

If each user $\user \in \intint{\nusers}$ only likes a single item category $s_\user \in \intint{S}$, i.e.  $\forall a \in \actS \setminus \actS_{s_\user}, \rho^\user(\act)= 0$, then $\frac{\sum\limits_{a\in\actSsmall_s}  \exprew^\user(a)}{\sum\limits_{a\in\actS}  \exprew^\user(a)} = \indic{s = s_\user}.$

Then $u^\user(\piee{\user}) = \rho^\user(a_{s_\user}^\user) = \max_{a \in \actS} \rho^\user(a).$

Then $\piee{\user}$ is the optimal unconstrained policy for user $\user$, meaning the whole system is envy-free (cf. Sec \ref{sec:envy-def}).

From Eq. \ref{eq:ee-util}, we actually note that $(\piee{\user})_{\user=1}^\nusers = (\pi^{\user,*})_{\user=1}^\nusers$ if and only if each user $\user$ equally values their favorite items in each category they like, i.e. $\forall \user, \, \exists \kappa > 0, \forall s \in S, \rho^\user(a^\user_s) > 0 \Rightarrow \rho^\user(a^\user_s) = \kappa.$%$\forall \user, \forall s,s' \in S, \rho^\user(a^\user_s) =  \rho^\user(a^\user_{s'}).$

\end{proof}

% \begin{proposition}\label{prop:envyEE} With the above notation:
% \begin{itemize}
%     \item the policies $(\pipar{\user})_{\user=1}^\nusers$ are envy-free, while
%     \item the policies $(\piee{\user})_{\user=1}^\nusers$ are not envy-free in general.
% \end{itemize}
% \end{proposition}
\section{Extension to group envy-freeness} \label{app:homogroups}

We briefly discuss an extension of envy-free recommendation to groups, since most of the literature on fair machine learning focuses on systematic differences between groups. Certifying envy-freeness at the level of groups rather than individuals also relaxes the criterion because it requires less exploration. 
Let us assume we are given a partition $G$ of the users into disjoint groups. For $g,g'\in G,$ we define the group utility of $g$ with respect to $g'$ as:
\begin{equation}\label{eq:defgrouputil}
    U(g,g') = \frac{1}{\card{g}}\sum_{\user \in g} \iutil^\user\bigg(\frac{1}{\card{g'}}\sum_{n\in 
    g'} \pi^n\bigg)\,.
\end{equation}
\begin{definition}Given $\epsilon\!\geq\! 0$, the recommender system is \emph{$\epsilon$-group-envy-free} if:
$\quad\forall g,g' \in G,\quad U(g,g') \leq U(g,g) + \epsilon\,.$ 
\end{definition}
Group envy-freeness is equivalent to envy-freeness when each group is a singleton. When we have prior knowledge that user preferences and policies are homogeneous within each group, $\epsilon$-envy-freeness translates to $\epsilon'$-group envy-freeness, with $\epsilon'\approx \epsilon$, and the reciprocal is also true:
\begin{proposition}\label{prop:GEFtoEF}
Let $\epsilon, \tilde{\epsilon}>0$, and assume that for all groups and all pairs of users $\user, \otheruser$ in the same group $\group$, we have $\sup\limits_{\ctx\in\ctxS} \norm{\pi^\user(.|\ctx) - \pi^\otheruser(.|\ctx)}_1 \leq \tilde{\epsilon}$ and $\sup\limits_{\ctx\in\ctxS} \norm{\exprew^\user(.|\ctx) - \exprew^\otheruser(.|\ctx)}_1 \leq \tilde{\epsilon}.$
Then, $\epsilon$-group envy-freeness implies $(\epsilon+4\tilde{\epsilon})$-envy-freeness.
\end{proposition}
The result is natural since when all groups have users with homogeneous preferences and policies, groups and users are a similar entity as regards the assessment of envy-freeness. The proof is straightforward and omitted. 
When groups have heterogeneous policies, the ``average policy'' $\frac{1}{\card{g}}\sum_{n\in 
    g} \pi^n$ is uninformative because it does not represent any user's policy. Defining a notion of group utility in the general case is thus nontrivial and left for future work.
\section{Sources of envy}\label{app:sources}

In this section, we first list a few possible sources of envy in recommender systems. Then we provide the details of experiments\footnote{For all our experiments, we used Python and a machine with Intel Xeon Gold 6230 CPUs, 2.10 GHz, 1.3 MiB of cache.} which showcase one of these sources, namely model mispecification (App. \ref{sec:env-mispec}).

\subsection{Examples of sources of envy}

\paragraph{Model mispecification} Recommender systems often rely on strong modeling assumptions and multi-task learning, with methods such as low-rank matrix factorization \cite{koren2009matrix}. The limited capacity of the models (e.g., a rank that is too low) or incorrect assumptions might leave aside users with less common preference patterns. Appendix \ref{sec:env-mispec} gives a more detailed example  on two simulated  recommendation tasks.

\paragraph{Misaligned incentives} A recommender system might have incentives to recommend some items to specific users, e.g., sponsored content. Envy appears when there is a mismatch between users who like these items and users to whom they are recommended.% between the two types of users.

\paragraph{Measurement bias} Many hybrid recommender systems rely on user interactions together with user-side data \cite{burke2002hybrid}. This includes side-information such as browsing history on third-party, partner websites. Envy arises in these settings if there is measurement bias \cite{suresh2019framework}, e.g., if the side information is unevenly collected for all users (e.g., browsing  patterns are different across users and partners are aligned with the patterns of a user groups only). 

\paragraph{Operational constraints} Regardless of incentives, recommendations might need to obey additional constraints. As described in Proposition \ref{prop:envyEE}, the item-side fairness constraint of equity of exposure is an example of possible source of (user-side) envy. The user-side fairness constraint of equal utility also creates envy, as we showed in Sec. \ref{sec:exp-envy}.

\paragraph{}In the following, we provide the details of our experiments from Sec. \ref{sec:exp-envy} where we showcase examples of environments with envy based on movie and music recommendation tasks.

In these experiments, we measure envy based on the quantity:\begin{align*}\degenvy^m = \max\big(\max\limits_{n\in \intint{\nusers}}\iutil^m(\pi^n)- \iutil^m(\pi^m), 0\big)\end{align*} In line with \cite{chevaleyre2017distributed}, we consider two ways of measuring the degree of envy:
\begin{itemize}[leftmargin=*]
    \item the average envy experienced by users: $\frac{1}{\nusers}\sum\limits_{m\in \intint{\nusers}}\degenvy^m $,
    \item the proportion of $\epsilon$-envious users: $\frac{1}{\nusers}\sum\limits_{m\in \intint{\nusers}}\indic{\degenvy^m > \epsilon}$.
\end{itemize}

\subsection{Setup of the experiments on envy from model mispecification} \label{sec:env-mispec}

We describe in this section the details of the experiments on envy from mispecification presented in Section \ref{sec:exp-envy}. We used Lastfm-2k \citep{Cantador:RecSys2011}, a dataset from the online music service Last.fm\footnote{\url{http://www.lastfm.com}} which contains real play counts of $2k$ users for $19k$ artists, and was used by \citet{patro2020fairrec} who also study envy-freeness as a user-side fairness criterion. We filter the top $2,500$ items most listened to. Following \cite{johnson2014logistic}, we pre-process the raw counts with $\log$-transformation. We split the dataset into train/validation/test sets, each including $70\%/10\%/20\%$ of the user-item listening counts. We create three different splits using three random seeds. We estimate relevance scores for the whole user-item matrix using the standard matrix factorization algorithm\footnote{Using the Python library Implicit: \url{https://github.com/benfred/implicit} (MIT License).} of \citet{hu2008collaborative} trained on the train set, with hyperparameters selected on the validation set by grid search with DCG@40 as metric. The number of latent factors is chosen in $[16, 32, 64, 128]$, the regularization in $[0.01, 0.1, 1., 10.]$, and the confidence weighting parameter in $[0.1, 1., 10., 100.]$. The resulted matrix of estimated relevance scores serves as the ground truth preferences.

We also address movie recommendation using the MovieLens-1M dataset \cite{harper2015movielens}, which contains 1 million ratings on a 5-star scale from approximately 6000 users and 4000 movies. We extract a $2000\times2500$ user $\times$ items matrix, keeping users and items with the most rating. We transform MovieLens ratings into an implicit feedback dataset similar to Last.fm. Since setting ratings $<3$ are usually considered as negative \cite{wang2018modeling}, we set ratings $<3$ to zero, resulting in a dataset with preference values among $\{0,3,3.5,4,4.5,5\}$. We then use the same algorithm as for Last.fm to obtain relevance scores that we use to simulate ground truth preferences.

We then simulate a recommender system's estimation of preferences using low-rank matrix completion\footnote{Using the implementation of \url{https://github.com/gbolmier/funk-svd} (MIT License).} \cite{bell1995information} on a training sample of $70\%$ of the whole ``ground truth'' preferences, with hyperparameter selection on a $10\%$ validation sample. Here, the regularization is chosen in $[0.001, 0.01, 0.1, 1.]$, and the confidence weighting parameter in $[0.1, 1., 10., 100.]$. The estimated preference scores are given as input to the recommendation policies.

The recommendation policies we consider are softmax distributions over the predicted scores with fixed inverse temperature. These policies recommend a single item, drawn from the softmax distribution.

We generate binary rewards using a Bernoulli distribution with expectation given by our ground truth. We consider no context in these experiments, so that the policies and rewards only depend on the user and the item. 

Figure \ref{fig:mispeargmaxageocc} in Sec. \ref{sec:exp-envy} was generated by varying the number of latent factors in the recommender system's preference estimation model. For each number of latent factors in the range $[1,2,4,8,16,32,64,128,256]$, a new model was trained on the train set with hyperparameter selection on the validation set. The degrees of envy are measured on the whole ground truth preference matrix.

\subsection{Envy from equal user utility constraints} \label{app:envy-euu}
We provide the full details of the experiments on envy from equal user utility presented in Sec. \ref{sec:exp-envy} from the main paper. The goal of these experiments is to show that in contrast to envy-freeness, enforcing equal user utility (EUU) degrades user satisfaction and creates envy between users. We remind from Sec. \ref{sec:envy-def} that the fairness constraint of EUU is defined as:
\begin{align*}
    \forall \user,\otheruser \in \intint{\nusers}, \iutil^\user(\pi^\user) = \iutil^\otheruser(\pi^\otheruser),
\end{align*}
or equivalently:
\begin{align*}
    \forall \user \in \intint{\nusers}, \iutil^\user(\pi^\user) = \frac{1}{\nusers}\sum_{\otheruser \in \intint{\nusers}}\iutil^\otheruser(\pi^\otheruser).
\end{align*}
Equal user utility is enforced by adding a penalty to the maximization of user utilities. Optimal EUU policies are found by maximizing the following concave objective function, where the parameter $b>0$ controls the strength of the penalty: 

\begin{align}
&\text{\emph{(EUU)}}  \quad  %\forall \user\in\intint{\nusers}, \forall \ctx\in\ctxS:  
    \pieuu = \argmax_{\substack{p:\actS\to[0,1]^M\\\forall \user, \sum_a p^m(a) = 1}}  \sum_{\user \in \intint{\nusers}}\iutil^\user(p^m) - \euureg \sqrt{D(p)}\nonumber\\ 
    &
\quad\text{with }\quad  D(p) = \sum_{\user \in \intint{\nusers}}\bigg(u^m(p^m) - \frac{1}{M}\sum_{\otheruser \in \intint{\nusers}} u^n(p^n) \bigg)^2.\label{eq:EUU-obj}
\end{align}

We infer EUU policies using the Frank-Wolfe algorithm \cite{frank1956algorithm} with the ground truth preferences given as input. The parameter of the penalty is set to $b=50.$ We also generate the unconstrained optimal policies (OPT) based on the ground truth (recall that these are $\iutil^\user(\piu{\user}) = \max_{\pi} \iutil^\user(\pi) \geq  \iutil^\user(\piu{\otheruser})$).

A comparison of EUU and OPT is provided in Table \ref{tab:euu-bad} in Sec. \ref{sec:exp-envy}, with the following evaluation measures : total utility (higher is better), average envy and proportion of $0.05$-envious users (lower is better). The results on both dataset confirm the claim that enforcing EUU penalties deteriorates total utility and creates envy between users, while illustrating the known property that OPT policies are compatible with envy-freeness.
\section{\banditalg experiments}
\subsection{Bandit experiments}\label{sec:banditexps}

\begin{figure}[t]
    \centering
    \includegraphics[width=\linewidth]{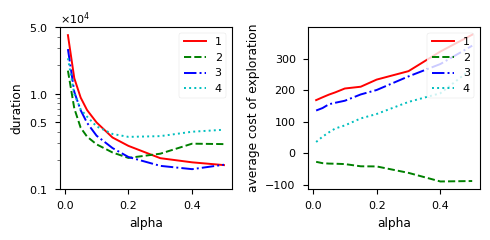}
    \caption{Effect of the conservative exploration parameter $\alpha$ on the duration and cost of auditing on Bandit experiments.\label{fig:banditexp2}}
\end{figure}

We performed experiments on toy bandit environments to assess the performance of our algorithm \banditalg on various configurations, which were also considered in \cite{jamieson2014best}. The four bandits instances have 10 arms. They are Bernoulli variables with means equal to
\begin{enumerate}[label=\arabic*),leftmargin=*]
    \item $\util_0 = 0.6$ and $\util_\grarm = 0.3$ for $k\in\intint{9}$,
    \item $\util_0 = 0.3$, $\mu_1 = 0.6$ and $\util_\grarm = 0.3$ for $k=2..9$,
    \item $\mu_k = 0.7 - 0.7*\big(\frac{k}{10}\big)^{0.6}, k = 0, ..., 9$, and the baseline is $\mu_0$,
    \item same as 3), but permuting $\util_0$ and $\util_1$.
\end{enumerate}

Fig.~\ref{fig:banditexp2} shows the result of applying \banditalg on the various configurations, where we set $\delta=\epsilon=0.05$, $\omega=0.99$ \footnote{Following \citep{jamieson2014lil} who recommend $\omega$ close to 1.} and report results averaged over $100$ trials. We observe clear tendencies similar to those presented in Section \ref{sec:exp-ocef}, although the exact sweet spots in terms of $\alpha$ depends on the specific configuration.

The cost of exploration follows similar patterns as in in Section \ref{sec:exp-ocef}. In Prob. 2, the baseline has the worst performance, so exploration is beneficial to the user and the cost is negative. On the other hand, for instance in Prob. 4, the cost is close to $0$ when $\alpha$ is very small and increases with $\alpha$. It is the case where the baseline is not the best arm but is close to it, and there are many bad arms. When the algorithm is very conservative, bad arms are discarded rapidly thanks to the good estimation of the baseline performance. In this ``low-cost'' regime however, the audit is significantly longer.

We show additional results when varying $\delta$ in Figure~\ref{fig:banditexpdelta}. Results are averaged over $100$ simulations and the conservative exploration parameter is set to $\alpha=0.05$. The duration decreases as $\delta$ increases, i.e. a lower confidence certificate requires fewer samples per user. The duration for Problem 1 is longer than for the other instances. This is because with $\alpha$ set to $0.05$ and the baseline mean being much higher than non-baseline arms, the conservative constraint \ref{eq:c-constraint} enforces many pulls of the baseline, since each exploration round is very costly. As a consequence, too little data is collected on the non-baseline arms to conclude that they are below $\util_0 + \epsilon$. Since all non-baseline arms have equal means, the size of the active set remains the same for a long time, while in Problem 3, where the baseline is also the best arm, arms are eliminated one at a time.

We show how \banditalg scales with the number of arms in Figure~\ref{fig:banditexparms}, for fixed values $\alpha = \delta = \epsilon = 0.05$. We set $K_{\max} = 100$ and define 4 instances as in the list above, except that $K=K_{\max}$ instead of $K=9$. We run \banditalg on the instances $\util_{0:K'}$ and vary the value of $K' \leq K_{\max}$. The duration increases for all problems, and the slope depends on the gaps between $\util_0$ and the $\util_\grarm$.

\begin{figure}[t]
    \centering
    \includegraphics[width=\linewidth]{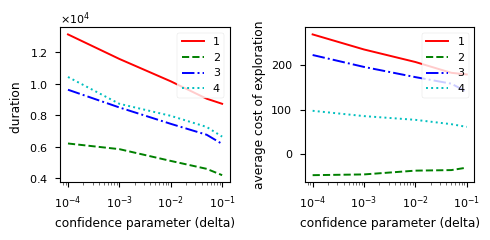}
    \caption{Effect of the confidence parameter $\delta$ on the duration and cost on 4 different bandit instances.\label{fig:banditexpdelta}}
\end{figure}

\begin{figure}[t]
    \centering
    \includegraphics[width=\linewidth/2]{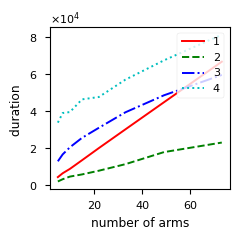}
    \caption{Effect of the number of arms on the duration on 4 different bandit instances.\label{fig:banditexparms}}
\end{figure}

\subsection{Setup of the MovieLens and Last.fm experiments}\label{sec:movielensexps}

We now provide additional details on the experimental evaluation of \banditalg on MovieLens and Last.fm presented in Sec. \ref{sec:exp-ocef}. The protocole to generate the recommendation task is the same as the one described in App. \ref{app:sources} for the experiments on sources of envy. The policies are softmax distributions over scores predicted by the matrix factorization model with a number of factors equal to $48$.

In these experiments, the auditor interacts with the audited users. Rewards are drawn from Bernoulli distributions with expectation equal to the ground truth preferences.

Two recommendation policies are audited. The first one is a softmax with inverse temperature equal to $5$. Since the inverse temperature is small, the softmax distribution is closer to random, which means users get more similar recommendations: the recommender system is thus envy-free. The second one is a softmax with inverse temperature equal to $15$. With higher inverse temperature, the distribution is more peaked, which exacerbates differences between policies. Since the model with $48$ factors is mispecified (see Sec.\ref{sec:exp-envy}), envy is visible.

\section{Proofs}\label{sec:proofs}

\subsection{Theoretical results}

\subsubsection{Useful lemmas}\label{sec:lemma}

Recall that \banditalg considers a single audited group $\user$, therefore we do not use superscripts $\user$ in the following (e.g., $\mu_k, r_t$...).

The algorithm relies on valid confidence intervals. As in \cite{jamieson2014lil}, we use anytime bounds inspired by the law of the iterated algorithm (LIL), and a union bound. 

We say that a random variable is $\sigma$-subgaussian if it is subgaussian with variance proxy $\sigma^2$. Since we assume the rewards for each user are bounded, more precisely $r_t\in [0,1]$, they are $\frac{1}{2}$-subgaussian.

Throughout the paper, we assume that rewards for each user are independent conditionally to the arm played.

\begin{lemma}\label{lem:chernoff2} Let $\conf \in (0,1)$. Assume the rewards are $\sigma$-subgaussian.

Let $\omega \in (0,1), \quad\theta = \log(1+\omega)\big(\frac{\omega\conf}{2(2+\omega)}\big)^{\frac{1}{1+\omega}}$. 
\begin{align*}
    \text{Let}\quad N_\grarm(t) = \sumt \indic{\grarm_s = \grarm} && \empEst_\grarm(t) = \frac{\sumt \rew_s \indic{\grarm_s = \grarm}}{ N_\grarm(t)}
\end{align*}
\begin{align*}
    \rad_\grarm(t) = &\sqrt{\frac{2\sigma^2(1+\sqrt{\omega})^2(1+\omega)}{\cnt_\grarm(t)}}\\
    &\times \sqrt{\log\left(\frac{2(\narms+1)}{\theta}\log((1+\omega)\cnt_\grarm(t))\right)}
\end{align*}
\begin{align*}
\lcb_k(t) = \empEst_k(t) - \rad_k(t) &&
\ucb_k(t) = \empEst_k(t) + \rad_k(t) 
\end{align*} 
Then, %$\forall \omega \in (0,1), \forall \theta\in(0,\frac{\log(1+\omega)}{e})$, 
\begin{align*} \pr{}{\forall t>0, \forall \grarm\in\intint{\narms}, ~~\util_\grarm\in[\lcb_\grarm(t); \ucb_\grarm(t)]} \geq 1- \frac{\delta}{2}\,.
\end{align*}
\end{lemma} 

Notice that the choice of $\theta$ makes sure that $\rad_\grarm$ is well defined as long as $\cnt_\grarm(t) > 0$. We use the convention that when $\cnt_\grarm(t) = 0$, $\rad_\grarm(t)$ is strictly larger than when $\cnt_\grarm(t) = 1$ to ensure $\rad_\grarm$ is strictly decreasing with $\cnt_\grarm$. Also, when $\cnt_\grarm(t) = 0$, we set $\empEst_\grarm(t) = 0$.

Following \cite{garcelon2020conservative}, our lower bound on the conservative constraint relies on Freedman's martingale inequality \cite{freedman1975tail}.

\begin{lemma}\label{lem:conf-interval}Assume all rewards are $\sigma$-subgaussian. Let $\activepullS_{t}=\{s\leq t: \grarm_s \neq 0\}$ be the number of times a non-baseline arm $\grarm \neq 0$ has been pulled up to time $t$. Let
$\phi(t) = \sigma\sqrt{2\card{\activepullS_{t-1}}\log\big(\frac{6 \card{\activepullS_{t-1}}^2}{\delta}\big)} + \frac{2}{3}\log\big(\frac{6 \card{\activepullS_{t-1}}^2}{\delta}\big).$%, and let $\Phi(t) = \min\big(\sumArm \rad_\grarm(t-1) \cnt_\grarm(t-1), \phi(t)\big).$
 
 Then,$\quad\forall \conf > 0,$\begin{align*}
     \pr{}{\forall t>0, \bigg\vert\sum_{s\in\activepullS_{t-1}} (\util_{\grarm_s} - \rew_s) \bigg\vert \leq  \phi(t)} \geq 1-\frac{\conf}{2}.
\end{align*}
\end{lemma}

As in Lemma \ref{lem:chernoff2}, we use the convention $\phi(t)= 0$ when $\card{A_{t-1}}=0$.

\begin{lemma}\label{lem:bound-constraint} Let $\conf \in (0, 1)$. 

Let $\Phi(t) = \min\big(\sumArm \rad_\grarm(t-1) \cnt_\grarm(t-1), \phi(t)\big)$, with $\phi(t)$ defined in Lemma \ref{lem:conf-interval}. Let $\mathcal E$ be the event under which all confidence intervals are valid, i.e.:
\begin{equation*}
\begin{aligned}
    &\mathcal{E} = \mathcal{E}_1\cap \mathcal{E}_2 \quad \, \text{ with}\\
    &\mathcal{E}_1 = \big\{\forall k\in\{0, \ldots, K\}, \forall t>0, \mu_k(t)\in[\lcb_k(t); \ucb_k(t) \big\}\\
    &\mathcal{E}_2 = \big\{\forall t>0, \bigg\vert\sum_{s\in\activepullS_{t-1}} (\util_{\grarm_s} - \rew_s) \bigg\vert \leq  \Phi(t) \big\}.
\end{aligned}
\end{equation*}
Then $\pr{}{\mathcal{E}} \geq 1 - \delta$.
\end{lemma}
\begin{proof}
By Lemma \ref{lem:chernoff2}, $\pr{}{\mathcal{E}_1} \geq 1 -\frac{\delta}{2}$.
%By Freedman's martingale inequality \cite{freedman1975tail} and a union bound over timesteps,
By the lemma above, with probability $1-\frac{\delta}{2}$, we have for all $t>0$, $\big\vert\sum_{s\in\activepullS_{t-1}} (\util_{\grarm_s} - \rew_s) \big\vert \leq  \phi(t).$

Then, notice that \begin{align*}
    \bigg\vert\sum_{s\in\activepullS_{t-1}} (\util_{\grarm_s} - \rew_s)\bigg\vert = \bigg\vert\sumArm \cnt_\grarm(t-1) (\util_\grarm - \empEst_\grarm(t-1))\bigg\vert.
\end{align*}

Hence under $\mathcal{E}_1$ we also have: 
 \begin{align*}
    \bigg\vert\sum_{s\in\activepullS_{t-1}} (\util_{\grarm_s} - \rew_s)\bigg\vert \leq \sumArm \cnt_\grarm(t-1) \rad_\grarm(t-1).
\end{align*}

Therefore, $$\mathcal{E} = \mathcal{E}_1 \cap \mathcal{E}_2 =  \mathcal{E}_1 \cap \bigg\{\big\vert\sum_{s\in\activepullS_{t-1}} (\util_{\grarm_s} - \rew_s) \big\vert \leq  \phi(t)\bigg\}, $$ and thus, by a union bound, we have: $\pr{}{\mathcal E} \geq 1-\conf$.
\end{proof}

\subsubsection{Theorems}

We now provide our complete theoretical guarantees for correctness (Theorem \ref{th:banditcorrectness}), duration (Theorem \ref{th:banditsample}) and cost (Theorem \ref{th:costexp}), which we then prove in App. \ref{app:proof-correct} and \ref{app:proof-sample-cost}. From these results, we derive Theorem \ref{thm:all} in the main paper, which we prove in App. \ref{app:proof-all}. 

\begin{theorem}[Correctness] \label{th:banditcorrectness}
With probability at least $1-\delta$:
\begin{enumerate}
    \item \banditalg satisfies the safety constraint \eqref{eq:c-constraint} at every time step,
    \item if \banditalg outputs $\epsilon$-\noenvy then the user $\user$ is not $\epsilon$-envious, and if it outputs \envy, then $\user$ is envious.
\end{enumerate}
\end{theorem}

We denote $\log^+(.) = \max(1, \log(.))$.

\begin{theorem}[Duration] \label{th:banditsample} Let $\eta_\grarm = \max(\util_\grarm-\util_0, \util_0 + \epsilon - \util_\grarm)$, $\conf \in (0, 1)$, $\theta = \log(2)\sqrt{\frac{\conf}{6}}$, and %$\theta\in(0,\frac{\log(2)}{e})$,     $\conf=\frac{3\theta}{\log(2)}$, 
\begin{align*} \forall \grarm\neq 0,\text{~~~} H_\grarm = 1 + \frac{64}{\eta_\grarm^2}
\log\bigg(\frac{2(\narms+1)\log^+\big(\frac{128(\narms+1)}{\theta \eta_\grarm^2}\big)}{\theta}\bigg),
\end{align*}

\begin{align*}
    H_0 =  \max\bigg(\max_{\grarm\in\intint{K}} &H_\grarm, \frac{6K+ 2}{\alpha\util_0}  \\
    &+ \sum_{k=1}^K \frac{256\log\left(\frac{2(\narms+1)\log(2H_\grarm)}{\theta}\right)}{\alpha \util_0 \eta_\grarm}\bigg).%\\
    % & + \frac{1}{\alpha \util_0}\sqrt{\frac{2}{\card{\group}}\sumArm H_\grarm \log\big(\frac{3}{\delta}(\sum_{k'=1}^{\narms} H_{k'})^2\big)} \\
    % &+ \frac{4}{3\alpha \util_0} \log\big(\frac{3}{\delta}(\sumArm H_\grarm)^2\big)\bigg).
\end{align*}

With probability at least $1-\conf$, \banditalg stops in at most $\tau$ steps, with 
\begin{equation*}
    \tau \leq \sum_{k=0}^K H_k\,.
\end{equation*}
\end{theorem}
Finally, we define the \emph{cost of exploration} as the potential reward lost because of exploration actions, in our case the cumulative reward lost, on average over users in the group:
\begin{align}\label{eq:costexp}
    \costexp_t = t\util_0 - \sum_{s=1}^t \util_{k_s}\,.
\end{align}
In the worst case, the following bound holds:
\begin{theorem}[Cost of exploration]\label{th:costexp} 
%Let $\conf \in (0, \frac{1}{2})$, $\theta = \log(2)\sqrt{\frac{\conf}{3}}$.
Under the assumptions and notation of Theorem \ref{th:banditsample}, let $\tau$ be the time step where \banditalg stops. With probability $1-\delta$, we have:
\begin{equation}
\begin{aligned}
     \costexp_\tau \leq \sum_{\grarm: \util_k<\util_0} (\util_0-\util_\grarm) H_k
    %\costexp_\tau \leq K + \sum_{\grarm: \util_k<\util_0}& \frac{64(\util_0-\util_\grarm)}{\card{\group}\eta_\grarm^2}\\&\times\log\left(\frac{4(\narms+1)}{\theta}\log^+\left(\frac{128(\narms+1)}{\theta\card{\group}\eta_\grarm^2}\right)\right).
\end{aligned}
\end{equation}
%where an upper bound on $\tau$ is given by Th.~\ref{th:banditsample} 
\end{theorem}

\paragraph{Certification of the exact criterion for all users }

The audit of the full system for the exact envy-freeness criterion consists in running \banditalg for every user. Since we are making multiple tests, we need to use a tighter confidence parameter for each user so that the confidence intervals simultaneously hold for all users. 
\begin{corollary}[Online %synchronous
certification]\label{cor:audit-sync} With probability at least $1-\delta$, running \banditalg simultaneously for all $\nusers$ users, each with confidence parameter $\delta' = \frac{\delta}{\nusers}$, we have:
\begin{enumerate}
    \item for all $\user \in [\nusers]$ \banditalg satisfies the constraints \eqref{eq:c-constraint},
    \item all users for which \banditalg returns $\epsilon$-NO ENVY are not $\epsilon$-envious of any other users, and all users for which \banditalg returns ENVY are envious of another user.
    \item For every user, the bounds on the duration of the experiment and the cost of exploration given by Theorems \ref{th:banditsample} and \ref{th:costexp} (using $\delta/\nusers$ instead of $\delta$) are simultaneously valid.
\end{enumerate}
\end{corollary}

For the certification of the probabilistic envy-freeness criterion, we refer to Theorem \ref{th:audit-proba} in the main paper, which we prove in App. \ref{app:proof-audit-proba}.

\subsection{Proof of Theorem \ref{th:banditcorrectness}}\label{app:proof-correct}

\begin{proof}
We assume that event $\event$ holds true. Then all confidence intervals are valid, i.e., for all $\grarm = 0,..., \narms$, $\lcb_\grarm(t) \leq \util_\grarm \leq \ucb_\grarm(t)$, and $\sumActive\util_{\grarm_s} \geq \sumActive \rew_s - \Phi(t) $. 

Let $Z_t$ be the safety budget, defined as $Z_t = \sum_{s=1}^t \util_{\grarm_s} - (1 - \alpha) \util_0 t$, so that the conservative constraint \eqref{eq:c-constraint} is equivalent to $\forall t, Z_t \geq 0$. We have $Z_t = \sumActive \util_{\grarm_s} + \util_{\grarm_t} +  (N_0(t-1)-(1 - \alpha)t) \util_0 $.
Therefore, $\xi_t$ (eq. \eqref{eq:xi}) is a lower bound on the safety budget $Z_t$ if $\ell_t$ is played. By construction of the algorithm, the safety constraint \eqref{eq:c-constraint} is immediately satisfied since a pull that could violate it is not permitted.

By the validity of confidence intervals under $\event$, if \banditalg stops because of the first condition, then $\exists \grarm, \util_\grarm > \util_0$. Therefore $0$ is not $\epsilon$-envious of $\grarm$ and \banditalg is correct.

If \banditalg stops because of the second condition, i.e.,  $\forall \grarm, \ucb_\grarm(t) \leq \lcb_0(t) + \epsilon$, then $\forall \grarm, \util_\grarm \leq \util_0 + \epsilon$.  Therefore $0$ is not envious and \banditalg is correct.

Since $\pr{}{\event} \geq 1-\delta$, \banditalg satisfies the safety constraint and is correct with probability $\geq 1 - \delta$.

\end{proof}

\subsection{Proofs of Theorem \ref{th:banditsample} and Theorem \ref{th:costexp}}\label{app:proof-sample-cost}

\paragraph{Notation} For conciseness, we use $\tilde{\narms} = \narms+1$, and
\begin{equation*}
\begin{aligned}
    &\psi_\grarm(t) = 2\sigma^2(1+\sqrt{\omega})^2(1+\omega)\log\left(\frac{2\tilde{\narms}}{\theta}\log((1+\omega)\cnt_\grarm(t))\right),\\
    &~~~\text{so that~} \rad_k(t) = \sqrt{\frac{\psi_k(t)}{N_\grarm(t)}}.
\end{aligned}
\end{equation*}

We shall also use $\Gamma_\omega = 2\sigma^2(1+\sqrt{\omega})^2(1+\omega)$. We use the convention $\psi_\grarm(t) = 0$ when $\cnt_\grarm(t) = 0$, and set  $\rad_\grarm(t)$ to some value strictly larger than when $\cnt_\grarm(t) = 1$.

We remind that $\omega \in (0,1), \quad\theta = \log(1+\omega)\big(\frac{\omega\conf}{2(2+\omega)}\big)^{\frac{1}{1+\omega}}$ and $\eta_\grarm = \max( \util_\grarm-\util_0, \util_0 + \epsilon - \util_\grarm)$. We denote by $\eta_{\min} = \min_{\grarm\in\intint{\narms}} \eta_\grarm$.

Finally, we notice that under event $\mathcal E$ (as defined in Sec.~\ref{sec:lemma}), we have for all $\grarm \in \{0,\ldots, \narms\}$ and all $t$: 
\begin{equation}\label{eq:intervals}
    \util_\grarm + 2 \rad_\grarm(t) \geq \ucb_\grarm(t) \geq \util_\grarm \geq \lcb_\grarm(t) \geq \util_\grarm - 2\rad_\grarm(t). 
\end{equation}

\begin{lemma}\label{lem:rad} Under event $\event$, for every $\grarm\in\intint{\narms}$, if $\grarm$ is pulled at round $t$, then $4 \rad_\grarm(t) \geq \eta_\grarm$.
\end{lemma}
\begin{proof}[Proof of Lemma \ref{lem:rad}]
Since $\grarm$ is pulled at $t$, the two following inequalities hold:
\begin{align}\label{eq:last1}
    \ucb_\grarm(t-1) > \lcb_0(t-1) + \epsilon
\end{align}
\begin{align}\label{eq:last2}
    \lcb_\grarm(t-1) \leq \ucb_0(t-1) 
\end{align}
We prove them by contradiction. If \eqref{eq:last1} does not hold, then $\grarm$ should be discarded from the active set at time $t-1$, and therefore cannot be pulled at $t$. Likewise, if \eqref{eq:last2} does not hold, then the algorithm stops at $t-1$, so $\grarm$ cannot be pulled at $t$.

Using \eqref{eq:last1} and \eqref{eq:intervals}, we have: 
\begin{equation*}
    \util_\grarm + 2 \rad_\grarm(t-1) \geq \ucb_\grarm(t-1) > \lcb_0(t-1) + \epsilon \geq \util_0 - 2\rad_0(t-1) + \epsilon.
\end{equation*}
Since $0$ was not pulled at time $t$, we also have $\rad_0(t-1) \leq \rad_\grarm(t-1)$, hence $4 \rad_\grarm(t-1) \geq \util_0 + \epsilon - \util_\grarm.$

Using \eqref{eq:last2} and \eqref{eq:intervals} we have  $
    \util_k - 2\rad_k(t) \leq \util_0 + 2\rad_0(t)$ and since $\rad_0(t) \leq \rad_k(t)$, we obtain $4 \rad_\grarm(t-1) \geq \util_\grarm - \util_0$.
    
%Finally, substracting \eqref{eq:last2} from \eqref{eq:last1}, we get $4\rad_\grarm(t-1) > \epsilon - 4\rad_0(t-1)$. Again using $\rad_0(t-1) \leq \rad_\grarm(t-1)$, we obtain $4 \rad_\grarm(t-1) > \frac{\epsilon}{2}$.
\end{proof}
In the following lemma, we recall that we denote $\log^+(.) = \max(1, \log(.))$.
\begin{lemma}\label{lem:counts} Under event $\mathcal E$, $\forall \tau>0, \forall \grarm\in\intint{\narms}$, we have
\begin{align*}
    &N_k(\tau) \leq H_k \text{ ~~~~ with} \\
    &H_\grarm = 1+
    \frac{32\sigma^2(1+\sqrt{\omega})^2(1+\omega)}{\eta_\grarm^2}\times\\
    &\log\bigg(\frac{2(\narms+1)\log^+\big(\frac{64(\narms + 1)\sigma^2(1+\sqrt{\omega})^2(1+\omega)^2}{\theta\eta_\grarm^2}\big)}{\theta}\bigg)
\end{align*}
\end{lemma}
\begin{proof}
Let $\tau>0$, $\grarm\in\intint{\narms}$, and let $t\leq \tau$ be last time step before $\tau$ at which $\grarm$ was pulled. If such a $t$ does not exist, then $N_\grarm(\tau) = 0$ and the result holds. In all cases, we have $N_\grarm(t) = N_\grarm(\tau)$. 

We consider $t>0$ from now on.

By Lemma \ref{lem:rad}, we have $4\rad_\grarm(t-1)\geq \eta_\grarm$, and thus $N_\grarm(t-1) \leq \frac{16\psi_k(t-1)}{\eta_k^2}$, which writes, if $\cnt_\grarm(t) >0$:
\begin{equation}\label{eq:NasPsi}
\begin{aligned}
      N_\grarm(t-1) &\leq \frac{16\psi_k(t-1)}{\eta_k^2}\\
    &\leq \frac{16 \Gamma_\omega}{\eta_\grarm^2}\log\left(\frac{2\tilde{\narms}}{\theta}\log\left((1+\omega)\cnt_\grarm(t-1)\right)\right)\,.  
\end{aligned}
\end{equation}
Using $ \frac{1}{t}\log\left(\frac{\log((1+\omega)t)}{\Omega}\right) \geq c \Rightarrow t \leq \frac{1}{c} \log\left(\frac{\log((1+\omega)/c\Omega)}{\Omega}\right)$ (see Equation (1) in \cite{jamieson2014lil}) with $\Omega = \frac{\theta}{2\tilde{\narms}}$ and $c = \frac{\eta_\grarm^2}{16 \Gamma_\omega}$, we obtain
\begin{equation}
    N_\grarm(t-1) \leq \frac{16 \Gamma_\omega}{\eta_\grarm^2} \log\big(\frac{2\tilde{\narms}}{\theta}\log\big( \frac{(1+\omega)32\tilde{\narms} \Gamma_\omega}{\theta \eta_k^2}\big) \big)
\end{equation}
Since $N_\grarm(t) = N_\grarm(t-1) + 1$, using $\log^+$ instead of $\log$ inside to deal with the case $\cnt_\grarm(t-1)=0$ gives the desired result.
\end{proof}

\begin{lemma}\label{lem:baseline-bound}
Under event $\event$, at every time step $\tau$, we have
\begin{align*}
    \cnt_0&(\tau)\leq \max\bigg(\max_{\grarm\in\intint{K}} H_\grarm, \frac{6K+ 2}{\alpha\util_0} \\
    & + \sum_{k=1}^K \frac{64 \sigma^2(1+\sqrt{\omega})^2(1+\omega)\log\left(\frac{2(\narms+1)\log((1+\omega)H_\grarm)}{\theta}\right)}{\alpha \util_0 \eta_\grarm}\bigg)%\\
    %& + \frac{2\sigma}{\alpha \util_0}\sqrt{\frac{2}{\card{\group}}\sumArm H_\grarm \log\big(\frac{3}{\delta}(\sum_{k'=1}^{\narms} H_{k'})^2\big)} \\
    %&+ \frac{4}{3\alpha \util_0} \log\big(\frac{3}{\delta}(\sumArm H_\grarm)^2\big)\bigg).
\end{align*}

\end{lemma}

\begin{proof} Let $\tau>0$ and $t\leq \tau$ the last time $0$ was pulled before $\tau$. We assume $t>0$.

\paragraph{Case 1:} $0$ was pulled because $\rad_0(t-1) > \min_{\grarm \in \intint{\narms}} \rad_\grarm(t-1)$.

Then $\cnt_0(\tau) = \cnt_0(t-1) + 1 
%< \max\limits_{\grarm \neq 0} \cnt_\grarm(t-1)
\leq 1+ \max\limits_{\grarm \neq 0} \cnt_\grarm(t-1) $.

By lemma \ref{lem:rad}, we thus have 
%$\cnt_0(\tau) \leq \frac{16 \psi(t-1)}{\card{\group}\eta_{\min}^2} + 1$ and 
$\cnt_0(\tau) \leq \max_{k\in\intint{\narms}} H_\grarm$.

\paragraph{Case 2:} $0$ was pulled because $\xi_t < 0$.
Here the proof follows similar steps as that of Theorem 5 in \cite{wu2016conservative}.
\begin{align*}
    \sumActive \rew_s - \Phi(t) + &\lcb_{\ell_t}(t{-}1)\\
    &+ (N_0(t{-}1) - (   1-\alpha) t) \ucb_0(t{-}1) < 0
\end{align*}
We drop $\lcb_{\ell_t}(t{-}1)$, replace $t$ by $\sum_{k=0}^{\narms} N_\grarm(t-1) + 1$ and rearrange terms to obtain:
\begin{align}\label{eq:cons1}
    \alpha N_0&(t-1) \ucb_0(t-1) \leq  (1 - \alpha ) \ucb_0(t-1) \nonumber\\ &+ ( 1-\alpha) \sumArm N_\grarm(t-1) \ucb_0(t-1) - \sumActive \rew_s + \Phi(t)
\end{align}
Since we have $\rad_0(t-1) \leq \rad_\grarm(t-1)$ (otherwise we would be in case 1), and $\activepullS_{t-1} = \sumArm \cnt_\grarm(t-1)$, we bound the the sum over arms in \eqref{eq:cons1}: 
\begin{align*}
\sumArm N_\grarm(t-1) &\ucb_0(t-1) \\
&\leq \sumArm N_\grarm(t-1) (\util_0 + 2\rad_0(t-1))\\
& \leq \sumArm N_\grarm(t-1) (\util_0 + 2\rad_\grarm(t-1))\\
& = \sumActive \util_0 + \sumArm 2 \rad_\grarm(t-1) N_\grarm(t-1).
\end{align*}

Using Lemma \ref{lem:bound-constraint}, we also bound $-\sumActive \rew_s  \geq \sumActive \util_s  + \Phi(t)$ (under $\mathcal{E}$).

Plugging this into \eqref{eq:cons1} gives:
\begin{align*}
    \alpha N_0(t-1) & \ucb_0(t-1) \leq  (1 - \alpha ) \ucb_0(t-1) \nonumber \\
    &+  2(1-\alpha)\sumArm N_\grarm(t-1) \rad_\grarm(t-1) \\& + \sumActive ((1-\alpha)\util_0 - \util_{\grarm_s}) + 2\Phi(t).
\end{align*}

Recall that $\Phi(t) = \min(\sumArm \cnt_\grarm(t-1)\rad_\grarm(t-1), \phi(t)),$ and therefore $\Phi(t) \leq \sumArm \cnt_\grarm(t-1)\rad_\grarm(t-1).$

Using $\util_0 - \util_{\grarm_s} \leq \eta_{\grarm_s}$ and $\sumActive \eta_{\grarm_s} = \sumArm \cnt_\grarm(t-1) \eta_{\grarm}$, we obtain:

\begin{align*}
    \alpha N_0(t-1) & \ucb_0(t-1) \leq  (1 - \alpha ) \ucb_0(t-1)  \\& + \sumArm \bigg((\eta_\grarm - \alpha \util_0)\cnt_\grarm(t-1) \\
    &+ 4\sqrt{\Psi_\grarm(t-1) \cnt_\grarm(t-1)} \bigg).
\end{align*}

We bound $\tmp_\grarm := (\eta_\grarm - \alpha \util_0)\cnt_\grarm(t-1) + 4\sqrt{\Psi_\grarm(t-1) \cnt_\grarm(t-1)}$.

Since  \eqref{eq:NasPsi} $N_k(t-1)\leq \frac{16\psi_k(t-1)}{\eta_k^2}+1$ , and $\eta_\grarm - \alpha \util_0 \leq \eta_\grarm$, we have 
\begin{align*}
    \tmp_k \leq \frac{16\psi_k(t-1)}{\eta_k} + \eta_k + 4\sqrt{\frac{16\psi_k(t-1)^2}{\eta_k^2} + \psi_k(t-1)}
\end{align*}
Using $\sqrt{(\frac{x}{z})^2 + x} \leq \frac{x}{z} + \frac{z}{2}$ for $x\geq0, z >0$, with $x = 4\psi_k(t-1)$ and $z = \eta_\grarm$, we obtain:
\begin{align}\label{eq:f1}
\tmp_\grarm &\leq \frac{16 \psi_k(t-1)}{\eta_\grarm} + \frac{16 \psi_k(t-1)}{\eta_\grarm} + 3\eta_k \nonumber\\
&\leq \frac{32 \psi_k(t-1)}{\eta_\grarm} + 3\eta_k.
\end{align}
Using $\psi_\grarm(t-1) = \Gamma_\omega\log\left(\frac{2\tilde{\narms}}{\theta}\log((1+\omega)\cnt_\grarm(t-1))\right)$ if $\cnt_\grarm(t)>0$ and $N_k(t-1) \leq H_k$ by Lemma \ref{lem:counts}, we obtain
\begin{align*}
    \tmp_\grarm \leq \frac{32 \Gamma_\omega}{\eta_\grarm}\log\left(\frac{2\tilde{\narms}}{\theta}\log\left((1+\omega)H_\grarm\right)\right) +3\eta_k\,.
\end{align*}
This bound is also valid when $\cnt_\grarm(t) >0$.

Going back to \eqref{eq:cons1}, and since $\util_0 \leq\ucb_0(t-1)$ under $\mathcal E$, we have (notice $\eta_k \leq 2$ since $\mu_\grarm\in[0,1]$ and $\epsilon\in[0,1]$):
\begin{equation}\label{eq:theend}
\begin{aligned}
     \alpha N_0(t-1) \util_0 \leq & (1 - \alpha ) \ucb_0(t-1) + 6K   \\
     & + \sumArm \frac{32 \Gamma_\omega}{\eta_\grarm}\log\left(\frac{2\tilde{\narms}}{\theta}\log\left((1+\omega)H_\grarm\right)\right).
\end{aligned}
\end{equation}

% We bound $2\Phi(t)$ using $\card{\activepullS_{t-1}} = \sumArm \cnt_\grarm(t-1) \leq \sumArm H_\grarm $ and obtain:
% \begin{align*}
%     \Phi(t) \leq \frac{\sigma}{\sqrt{\card{\group}}} \sqrt{2 \sumArm H_\grarm \log\left(\frac{3 (\sumArm H_\grarm)^2}{\delta}\right)} \\+ \frac{2}{3} \log\left(\frac{3 (\sumArm H_\grarm)^2}{\delta}\right)\,.
% \end{align*}

To bound the first term of the right-hand side, let us first notice that the final result holds if $\cnt_0(t-1) \leq \max_{k\in\intint{K}} H_k$. So we can assume $\cnt_0(t-1) > \max_{k\in\intint{K}} H_k$ from now on. By the definition of the $H_k$s (see above \eqref{eq:NasPsi}), this implies $\cnt_0(t-1) > \frac{16\psi_0(t-1)}{\eta_{\min}^2}$, which in turn implies $4\beta_0(t-1) \leq \eta_{\min}$.

We thus use $\ucb_0(t-1) \leq \util_0 + 2\rad_0(t-1) \leq \util_0 + \frac{\eta_{\min}}{2} \leq 2$, which gives the final result.

The result directly follows from \eqref{eq:theend}.
\end{proof}

The proof of  Theorem \ref{th:banditsample} follows from $\tau = \sumArm \cnt_\grarm(\tau) + \cnt_0(\tau)$, by setting $\omega=1$ for ease of reading, and $\sigma=\frac{1}{2}$ since Bernoulli variables are $\frac{1}{2}$-subgaussian (using Hoeffding's inequality \cite{hoeffding1963probability}).

We prove Corollary \ref{cor:audit-sync} from Theorem \ref{th:banditcorrectness} and Theorem \ref{th:banditsample}.

% \begin{proof} The sample complexity of \auditalg is a direct consequence of $T = \max_{\gr \in [\ngroups]} \tau^\gr.$
% \paragraph{Correctness} If \auditalg  outputs UNFAIR, it means that \banditalg output \envy for some $\gr \in [\ngroups]$, which is a correct prediction with probability $\geq 1 - \frac{\delta}{\ngroups} \geq 1 - \delta$. If \auditalg outputs $\epsilon-$FAIR, it means that for all $\gr \in [xxx]$, \banditalg output $\epsilon-$\noenvy, which is correct with probability $\geq 1 - \ngroups \frac{\delta}{\ngroups} = 1 - \delta$, by a union bound over all groups. Therefore, the certification is correct with probability $\geq 1 - \delta.$ 

% Moreover, for each group, the safety constraint \eqref{eq:c-constraint} is satisfied  by \banditalg with probability $\geq 1 - \frac{\delta}{\ngroups}$. By a union bound over groups, \auditalg satisfies the constraint \eqref{eq:c-constraint} for all groups with probability $\geq 1 - \delta.$
% \end{proof}

We now prove Theorem \ref{th:costexp}:
\begin{proof}
Since playing the baseline is neutral in the cost of exploration, it can be re-written as:
\begin{align*}
    \costexp_\tau = \sum_{\grarm=1}^\narms (\util_0 - \util_\grarm)N_\grarm(\tau) \leq \sum_{\grarm:\util_k<\util_0} (\util_0 - \util_\grarm)N_\grarm(\tau),
\end{align*}
where $\tau$ is the time the algorithm stops. Using Lemma \ref{lem:counts} to upper bound $N_\grarm(\tau)$, we obtain the result.
%under event $\mathcal E$:
% \begin{equation*}
% \begin{aligned}
%     \costexp_\tau &\leq \sum_{\grarm: \util_k<\util_0} (\util_0 - \util_\grarm)H_\grarm \\
%     &\leq K + \sum_{\grarm: \util_k<\util_0}  \frac{32(\util_0-\util_\grarm)\sigma^2(1+\sqrt{\omega})^2(1+\omega)}{\card{\group}\eta_\grarm^2}\times\\
%     &\log\left(\frac{4\tilde{\narms}}{\theta}\log\left(\frac{64\tilde{\narms}\sigma^2(1+\sqrt{\omega})^2(1+\omega)}{\theta\card{\group}\eta_\grarm^2}\right)\right).
% \end{aligned}
% \end{equation*}
% Again, we set $\omega=1$ and $\sigma=\frac{1}{2}$.
\end{proof}

Corollary \ref{cor:audit-sync} simply follows from the fact that by applying each algorithm with confidence $\delta/\nusers$, the confidence intervals are then simultaneously valid for all users with probability $1-\delta$, so all the correctness/duration/cost proofs holds for all groups simultaneously with probability $1-\delta$. For the statistical guarantees on certifying the probabilistic envy-freeness criterion, we provide the proof of Theorem \ref{th:audit-proba} in App. \ref{app:proof-audit-proba}.

\subsection{Proof of Theorem \ref{thm:all}}\label{app:proof-all}

Theorems \ref{th:banditcorrectness}, \ref{th:banditsample}, and \ref{th:costexp} are summarized in Theorem \ref{thm:all} in the main paper. We restate Theorem \ref{thm:all} and 
prove it below:%\footnote{The statement of the theorem is slightly different from the main paper. The main difference is the $\min$ instead of $+$ in the denominator for the duration, which corrects a typo. The other difference is the use of $h_k$, which yields the correct constants when $\card{\group}$ becomes large compared to $1/\eta_k$. The text under Theorem 1 in the main paper describes the correct statement (with the $\min$ instead of the $+$).}

\begin{theorem*}
Let $\epsilon\in(0,1]$, $\alpha\in(0,1], \delta\in(0,\frac{1}{2})$ and \begin{equation*}
    \eta_\grarm = \max(\util_\grarm-\util_0, \util_0 + \epsilon - \util_\grarm) \text{~and~} h_\grarm = \max(1, \frac{1}{\eta_\grarm}).
\end{equation*} 
Using $\lcb,\ucb$ and $\Phi$ given in Lemmas~\ref{lem:chernoff2} and \ref{lem:bound-constraint}, \banditalg achieves the following guarantees with probability at least $1-\delta$:
\begin{itemize}[leftmargin=*]
    \item \banditalg is correct and satisfies the conservative constraint on the recommendation performance \eqref{eq:c-constraint}.
    \item The duration is in
    $\displaystyle O\bigg(\sum_{k=1}^K\frac{h_\grarm \log\big(\frac{\narms\log(\frac{\narms h_\grarm}{\delta \eta_\grarm})}{\delta}\big)}{\min(\alpha\util_0,\eta_k)}\bigg)$.
    \item The cost is in $ O\bigg(\mathlarger{\mathlarger{\sum}}\limits_{k:\util_k<\util_0}\!\!\frac{(\util_0 - \util_k) h_\grarm}{\eta_k}\log\big(\frac{\narms\log(\frac{\narms h_\grarm}{\delta\eta_k})}{\delta}\big)\bigg)$.
\end{itemize}
\end{theorem*}

\begin{proof}
With $\delta \in (0, \frac{1}{2})$, let $\theta = \log(2)\sqrt{\frac{\conf}{6}}$. Then Theorems \ref{th:banditsample} and \ref{th:costexp} hold for $(\conf, \theta)$.

\paragraph{Duration} We first show that:
\begin{align}\label{eq:bound-Hk}
    H_\grarm = O\bigg(\frac{h_\grarm}{\eta_\grarm}\log\big(\frac{\narms h_\grarm}{\conf \eta_\grarm}\big)\bigg),
\end{align}\begin{align}\label{eq:bound-logHk}
    \log(H_\grarm) = O\bigg(\log\big(\frac{\narms h_\grarm}{\conf \eta_\grarm}\big)\bigg).
\end{align}

Recall from Th. \ref{th:banditsample} that $H_\grarm$ is defined as:
\begin{align*}
H_\grarm = 1 + \frac{64}{\eta_\grarm^2}
\log\bigg(\frac{2(\narms+1)\log^+\big(\frac{256(\narms+1)}{\theta\eta_\grarm^2}\big)}{\theta}\bigg)
\end{align*}

We replace the $\log^+$ term from Th. \ref{th:banditsample} by $\log\big(\frac{\narms h_\grarm}{\conf \eta_\grarm}\big) > 0$, because $\frac{\narms h_\grarm}{\conf} \geq 3$ as soon as $\narms \geq 2$. We thus have
\begin{align}\label{eq:intermHk}
    H_\grarm = 1 + O\bigg(\frac{1}{ \eta_\grarm^2} \underbrace{\log\bigg(\frac{\narms}{\conf} \log\big(\frac{\narms h_\grarm}{\conf \eta_\grarm}\big)}_{=B}\bigg)\bigg),
\end{align}

Using $\log(x) \leq x \Rightarrow x\log(x) \leq x^2$ for $x\geq 0$, and the fact that $\log\big(\frac{\narms h_\grarm}{\conf \eta_\grarm}\big) \geq 0$, we have:
\begin{align*}
    B \leq \log\bigg(\frac{\narms h_\grarm}{\conf \eta_\grarm}\log\big(\frac{\narms h_\grarm}{\conf \eta_\grarm}\big)\bigg) \leq 2\log\big(\frac{\narms h_\grarm}{\conf \eta_\grarm}\big).
\end{align*}
Since $1 + \frac{1}{\eta_\grarm^2} \leq 2\frac{h_\grarm}{\eta_\grarm},$ eq. \eqref{eq:bound-Hk} holds.

We now bound $\log(H_\grarm)$:
\begin{align}
    \log(H_k) &= O\bigg(\log\Big(\frac{h_\grarm}{\eta_\grarm}\log\big(\frac{\narms h_\grarm}{\conf \eta_\grarm}\big)\Big)\bigg)\\
    & = O\bigg(\log\Big(\frac{K h_\grarm}{\delta \eta_\grarm}\log\big(\frac{\narms h_\grarm}{\conf \eta_\grarm}\big)\Big)\bigg)\\
    & = O\bigg(\log\Big(\frac{K h_\grarm}{\delta \eta_\grarm}\Big)\bigg)
\end{align}
where the last line comes from $\frac{K h_\grarm}{\delta \eta_\grarm}\log\big(\frac{\narms h_\grarm}{\conf \eta_\grarm}\big) \leq \big(\frac{K h_\grarm}{\delta \eta_\grarm}\big)^2$.

% We now use this upper bound on $B$ to bound $\log(H_\grarm)$:
% \begin{align*}
%     \log(H_\grarm)&\leq \log\bigg(1 + \frac{2 \log\big(\frac{\narms h_\grarm}{\conf \eta_\grarm}\big)}{\card{\group}\eta_\grarm^2}\bigg)  \\
%     &\leq c \log\bigg(\frac{h_\grarm \log\big(\frac{\narms h_\grarm}{\conf \eta_\grarm}\big)}{\eta_\grarm}\bigg)\\
%     &\leq c' \log\big(\frac{\narms h_\grarm}{\conf \eta_\grarm}\big),
% \end{align*}
% where $c$ and $c'$ are some positive constants. The second inequality is because $\log\big(\frac{\narms h_\grarm}{\conf \eta_\grarm}\big) \geq 0$ and $h_\grarm \geq 1$. The last inequality uses the same argument as for the upper bound on $B$. 
Therefore, eq. \eqref{eq:bound-logHk} holds.

Now, let 
\begin{align*}
    \Gamma = \frac{6K+ 2}{\alpha\util_0} + \sum_{k=1}^K \frac{128\log\left(\frac{2(\narms+1)\log(2H_\grarm)}{\theta}\right)}{\alpha \util_0 \eta_\grarm},
\end{align*}
so that $H_0 = \max(\max_{\grarm\in \intint{\narms}} H_\grarm, \Gamma)$.

We have:
\begin{align*}
    \Gamma &= O\bigg(\frac{K}{\alpha \util_0} + \sumArm \frac{h_\grarm}{\alpha \util_0}\log\big(\frac{\narms\log(H_\grarm)}{\conf}\big)\bigg) \\
    &= O\bigg(\sumArm \frac{h_\grarm}{\alpha \util_0}\log\big(\frac{\narms\log(H_\grarm)}{\conf}\big)\bigg)\\
    &=O\bigg(\sumArm \frac{h_\grarm}{\alpha \util_0}\log\big(\frac{\narms\log(\frac{\narms h_\grarm}{\conf \eta_\grarm})}{\conf}\big)\bigg),
\end{align*}
where the second equality is because $\narms = \sumArm 1 \leq \sumArm h_\grarm,$ and the last equality uses eq. \eqref{eq:bound-logHk}. Combining this with eq. \eqref{eq:bound-Hk} we have:
\begin{align*}
    H_0 = O\bigg(\sumArm \frac{h_\grarm}{\min(\alpha \util_0, \eta_\grarm)}\log\big(\frac{\narms\log(\frac{\narms h_\grarm}{\conf \eta_\grarm})}{\conf}\big)\bigg).
\end{align*}
Using eq. \eqref{eq:bound-Hk} again to bound $\tau = H_0 + \sumArm H_\grarm,$ , we get the desired bound for duration.

\paragraph{Cost} For the cost, we remind the bound given in Th. \ref{th:costexp}:
\begin{equation}
\begin{aligned}
    \costexp_\tau &\leq \sum_{\grarm: \util_k<\util_0} (\util_0-\util_\grarm)H_k \\
    & = O\bigg(\sum_{\grarm: \util_k<\util_0} \frac{(\util_0-\util_\grarm)h_k}{\eta_k} \log\Big(\frac{\narms}{\conf} \log\big(\frac{\narms h_\grarm}{\conf \eta_\grarm}\big)\Big)\bigg)
\end{aligned}
\end{equation}
using \eqref{eq:intermHk} and $1+ \frac{1}{\eta_k^2} = O(\frac{h_k}{\eta_k})$.
% \begin{equation}
% \begin{aligned}
%     \costexp_\tau \leq \sum_{\grarm: \util_k<\util_0}& \frac{64(\util_0-\util_\grarm)}{\card{\group}\eta_\grarm^2}\\&\times\log\left(\frac{4(\narms+1)}{\theta}\log^+\left(\frac{128(\narms+1)}{\theta\card{\group}\eta_\grarm^2}\right)\right).
% \end{aligned}
% \end{equation}

% As for duration, the $\log^+$ term is in $O\big(\log\big(\frac{\narms h_\grarm}{\delta h_\grarm}\big)\big).$ The $\frac{1}{\card{\group}\eta_\grarm^2}$ factor in the sum is upper bounded by $\frac{h_\grarm}{\eta_\grarm}$. Hence we have:
% \begin{align*}
%     \costexp_\tau = O\bigg( K + \sum_{\grarm: \util_k<\util_0} \frac{(\util_0-\util_\grarm)h_\grarm\log\left(\frac{\narms}{\conf}\log\left(\frac{\narms}{\conf\card{\group}\eta_\grarm^2}\right)\right) }{\eta_\grarm}\bigg).
% \end{align*}

% As for duration, the final bound follows by observing that $\narms = \sumArm 1 \leq \sumArm h_\grarm$:
% \begin{align*}
%     \costexp_\tau = O\bigg(\sum_{\grarm: \util_k<\util_0} \frac{(\util_0-\util_\grarm)h_\grarm\log\left(\frac{\narms}{\conf}\log\left(\frac{\narms}{\conf\card{\group}\eta_\grarm^2}\right)\right) }{\eta_\grarm}\bigg).
% \end{align*}
\end{proof}

\subsection{Proof of Theorem \ref{th:audit-proba}}\label{app:proof-audit-proba}

We restate Theorem \ref{th:audit-proba} which summarizes the guarantees for the audit of the probabilistic envy-freeness criterion with \auditalg, and we 
prove it below:

\begin{theorem*}
Let $\efparams\in(0,1],\delta\in(0,\frac{1}{2})$. Let $\nsampled = \ceil{\frac{\log(3/\delta)}{\lambda}}$ and $\narms=\ceil{\frac{\log(3\nsampled/\delta)}{\log(1/(1-\gamma))}}$. With probability at least $ 1-\delta$,
\begin{itemize}[leftmargin=*]
\item \auditalg satisfies the conservative constraint \eqref{eq:c-constraint} for all $\nsampled$ audited users,
\item the bounds on duration and cost from Th.~\ref{thm:all} (using $\frac{\delta}{3\tilde{M}}$ instead of $\delta$) are simultaneously valid,
\item if \auditalg outputs $(\efparams)$-\envyfree, then the recommender system is $(\efparams)$-envy-free, and if it outputs \notenvyfree, then $\exists (\user,\otheruser),\, \iutil^\user(\pi^\user) < \iutil^\user(\pi^n)$.
\end{itemize}
\end{theorem*}

\begin{proof}
The first point is a consequence of Theorem \ref{th:banditcorrectness} and the second point is a consequence of Theorems \ref{th:banditsample} and \ref{th:costexp}. Since we apply \banditalg to each target user with confidence $\frac{\delta}{3 \nsampled}$, by the union bound the confidence intervals are simultaneously valid for all $\nsampled$ target users with probability $1-\frac{\delta}{3}$. Therefore, with probability at least $1-\frac{\delta}{3}$, the conservative constraint is satisfied for all $\nsampled$ users \emph{and} the bounds on cost and duration hold simultaneously for all $\nsampled$ users.

We now prove the third bullet point in two steps.
\paragraph{Step 1} We show that the value of $\narms=\frac{\log(3\nsampled/\delta)}{\log(1/(1-\gamma))}$ is chosen to guarantee the following result: with probability $1-\frac{\delta}{3\nsampled}$, if for a user we have $\mu_0 + \epsilon \geq \max\limits_{k\in \intint{\narms}}\mu_k$, then the user is not $(\epsilon,\gamma)$-envious.

First, we apply the theorem on random subset selection from (\citet{scholkopf2002learning}, Theorem 6.33), which guarantees that with probability $1 - (1-
\gamma)^\narms$, the arm with maximal reward among the $\narms$ arms is in the $(1-\gamma)$-quantile range of all possible $\nusers$ arms. Solving for $(1-
\gamma)^\narms = \frac{\delta}{3\nsampled}$, we get that when $\narms=\ceil{\frac{\log(3\nsampled/\delta)}{\log(1/(1-\gamma))}},$ the arm with maximal reward among the $\narms$ is in the $(1-\gamma)$ quantile range with probability $1 - \frac{\delta}{3\nsampled}.$  This means that if for a target user $\user$, we have $\iutil^\user(\pi^\user) + \epsilon = \mu_0 + \epsilon \geq \max\limits_{k\in \intint{\narms}}\mu_k,$ then with probability $1 - \frac{\delta}{3\nsampled}$, we also have:
\begin{align*}
    \mathbb{P}_{\otheruser \sim U_M}[\iutil^\user(\pi^\user) + \epsilon \geq \iutil^\user(\pi^\otheruser)] \geq 1-\gamma, 
\end{align*} meaning the user is not $(\epsilon,\gamma)$-envious. By a union bound over the $\nsampled$ target users, the property holds simultaneously for all $\nsampled$ target users with probability $1-\frac{\delta}{3}.$

\paragraph{Step 2} We now show that the number of users to audit $\nsampled=\ceil{\frac{\log(3/\delta)}{\lambda}}$ is chosen to guarantee that if none of the $\nsampled$ sampled users are $(\epsilon,\gamma)$-envious, then this holds true for an $(1-\lambda)$ fraction of the whole population with probability $1-\frac{\delta}{3}.$

Let $\delta'=\frac{\delta}{3}$. Denoting $q$ the probability that a user is not $(\epsilon,\gamma)$-envious, we want to guarantee that $q \geq 1-\lambda$ with probability at least $1-\delta'$, using $\nsampled$ Bernoulli trials where $p:= 1 - q$ is the probability of success. 

Let $\bar{B}(\nsampled, k, \delta')$ denote the largest $p'$ such that the probability of observing $k$ or more successes is at least $1-\delta'$ (i.e., $\bar{B}(\nsampled, k, \delta')$ is the binomial tail inversion). By definition, we have $p \leq \bar{B}(\nsampled, 0, \delta')$. Using the property that $\bar{B}(\nsampled, 0, \delta') \leq \frac{\log(1/\delta')}{\nsampled}$ (see e.g., \cite{langford2005tutorial}), we can guarantee that $p\leq \lambda$ as soon as  $\frac{\log(1/\delta')}{\nsampled} \leq \lambda$. Solving for $\nsampled$, we obtain that $\nsampled=\ceil{\frac{\log(1/\delta')}{\lambda}}=\ceil{\frac{\log(3/\delta)}{\lambda}}$ is sufficient to guarantee $p \leq \lambda$, or equivalenly $q\geq 1 - \lambda$ with probability $1-\frac{\delta}{3}.$

\paragraph{} We combining Step 1 and 2 by a union bound: if for $\nsampled$ users and $\narms$ arms, we have $\mu_0 + \epsilon \geq \max\limits_{k\in \intint{\narms}}\mu_k,$ then with probability $1 - \frac{2\delta}{3}$, an $(1-\lambda)$ fraction of the whole population is not $(\epsilon,\gamma)$-envious -- or equivalently, the recommender system is $(\efparams)$-envy-free. Since \banditalg is correct with probability $1-\frac{\delta}{3}$ when outputting that $\mu_0 + \epsilon \geq \max\limits_{k\in \intint{\narms}}\mu_k$ (i.e., $\epsilon$-\noenvy), the union bound guarantees with probability $1-\delta$ that \auditalg is correct when outputting $(\efparams)$-\envyfree.  Since \banditalg is correct with probability $\geq 1-\delta$ when outputting \envy, then so is \auditalg when outputting \notenvyfree, which achieves the proof of the third bullet point.

\end{proof}

\end{document}